\documentclass{article}

\usepackage{times}
\usepackage{graphicx} 
\usepackage{subfigure} 

\usepackage{natbib}

\usepackage{algorithm}
\usepackage{algorithmic}

\usepackage{hyperref}

\usepackage[accepted]{icml2016}

\usepackage{url}
\usepackage{tikz}
\usepackage{fancyhdr} 
\usepackage{color}
\usepackage{epsf}
\usepackage{epsfig}
\usepackage{bm}
\usepackage{amsmath, amssymb}
\usepackage{amsthm}
\usepackage{comment}
\usepackage{caption}
\usepackage{multirow}

\newtheorem{theorem}{Theorem}[section]
\newtheorem{definition}{Definition}[section]
\newtheorem{lemma}{Lemma}[section]
\newtheorem{remark}{Remark}[section]
\newtheorem{corollary}{Corollary}[section]

\def\layersep{2cm}
\def\layersepp{4cm}

\def\layersepppp{6cm}

\icmltitlerunning{Binary embeddings with structured hashed projections}

\begin{document} 

\twocolumn[
\icmltitle{Binary embeddings with structured hashed projections}

\icmlauthor{Anna Choromanska\thanks{1}}{achoroma@cims.nyu.edu}
\icmlauthor{Krzysztof Choromanski\footnotemark[1]}{kchoro@google.com}
\icmlauthor{Mariusz Bojarski}{mbojarski@nvidia.com}
\icmlauthor{Tony Jebara}{jebara@cs.columbia.edu}
\icmlauthor{Sanjiv Kumar}{sanjivk@google.com}
\icmlauthor{Yann LeCun}{yann@cs.nyu.edu}


\vskip 0.3in
]

\footnotetext[1]{Equal contribution.}

\begin{abstract} 
We consider the hashing mechanism for constructing binary embeddings, that involves pseudo-random projections followed by nonlinear (sign function) mappings. The pseudo-random projection is described by a matrix, where not all entries are independent random variables but instead a fixed ``budget of randomness'' is distributed across the matrix. Such matrices can be efficiently stored in sub-quadratic or even linear space, provide reduction in randomness usage (i.e. number of required random values), and very often lead to computational speed ups. We prove several theoretical 
results showing that projections via various structured matrices followed by nonlinear mappings accurately preserve the angular distance between input high-dimensional vectors. To the best of our knowledge, these results are the first that give theoretical ground for the use of general structured matrices in the nonlinear setting. In particular, they generalize previous extensions of the Johnson-Lindenstrauss lemma and prove the plausibility of the approach that was so far only heuristically confirmed for some special structured matrices. Consequently, we show that many structured matrices 
can be used as an efficient information compression mechanism. Our findings build a better understanding of certain deep architectures, which contain randomly weighted and untrained layers, and yet achieve high performance on different learning tasks. We empirically verify our theoretical findings and show the dependence of learning via structured hashed projections on the performance of neural network as well as nearest neighbor classifier.
\end{abstract} 

\section{Introduction}

The paradigm of binary embedding for data compression is of the central focus of this paper. 
The paradigm has been studied in some earlier works (see: \cite{weiss}, \cite{aqbe}, \cite{gong}, \cite{ssh}, \cite{gong2}, \cite{plan}, \cite{felix}, \cite{constantine}), 
and in particular it was observed that by using linear projections and then applying sign function as a nonlinear map 
one does not loose completely the information about the angular distance between vectors, but instead the information 
might be approximately reconstructed from the Hamming distance between hashes. In this paper we are interested in using 
pseudo-random projections via structured matrices in the linear projection phase. The pseudo-random projection is described 
by a matrix, where not all the entries are independent random variables but instead a fixed ``budget of randomness'' is 
distributed across the matrix. Thus they can be efficiently stored in a sub-quadratic or even linear space and provide 
reduction in the randomness usage. Moreover, using them often leads to computational speed ups since they provide fast 
matrix-vector multiplications via Fast Fourier Transform. We prove an extension of the Johnson-Lindenstrauss lemma~\cite{Sivakumar:2002:ADV:509907.509996} 
for general pseudo-random structured projections followed by nonlinear mappings. We show that the angular distance between high-dimensional 
data vectors is approximately preserved in the hashed space. This result is also new compared to previous extensions~\cite{journals/rsa/HinrichsV11,Vybíral20111096} 
of the Johnson-Lindenstrauss lemma, that consider special cases of our structured projections (namely: circulant matrices) and do not consider at 
all the action of the non-linear mapping. We give theoretical explanation of the approach that was so far only heuristically confirmed for some special structured matrices
(see: \cite{constantine}, \cite{felix}). 

Our theoretical findings imply that many types of matrices, such as circulant or Toeplitz Gaussian matrices, 
can be used as a preprocessing step in neural networks. Structured matrices were used before in different contexts also in deep learning, e.g.~\cite{ICML2011Saxe_551,sindhwani,cheng2,DBLP:journals/corr/MathieuHL13}). Our theoretical results however extend to more general class of matrices. 

Our work has primarily theoretical focus, but we also ask an empirical question: how the action of the random projection followed by non-linear transformation may influence learning? We focus on the deep learning setting, where the architecture contains completely random or pseudo-random structured layers that are not trained. 
Little is known from the theoretical point of view about these fast deep architectures, which achieve significant speed ups of computation 
and space usage reduction with simultaneous little or no loss in 
performance~\cite{ICML2011Saxe_551, cheng2, sindhwani, conf/iccv/JarrettKRL09,journals/ploscb/PintoDDC09,pinto:bionetics_2010, Huang2006489}. 
The high-level intuition justifying the success of these approaches is that not only does the performance of the deep learning system depend on learning, but also on the intrinsic properties of the architecture. These findings coincide with the notion of high redundancy in network parametrization~\cite{NIPS2013_5025,DBLP:journals/corr/DentonZBLF14,AChoro2015}. In this paper we consider a simple model of
the fully-connected feed-forward neural network where the input layer
is hashed by a structured pseudo-random projection followed by a point-wise nonlinear mapping.
Thus the input is effectively hashed and learning is conducted in the fully connected subsequent layers that act in the hashed space. 
We empirically verify how the distortion introduced in the first layer
by hashing (where we reduce the dimensionality of the data) affects the performance of the network (in the supervised learning setting). Finally, we show how our structured nonlinear embeddings can be used in the $k$-nn setting~\cite{citeulike:5847607}.

This article is organized as follows: Section~\ref{sec:rel} discusses related work, Section~\ref{sec:has_mech} explains the hashing mechanism, Section~\ref{sec:the} provides theoretical results, Section~\ref{sec:ne} shows experimental results, and Section~\ref{sec:con} concludes. Supplement contains additional proofs and experimental results.

\section{Related work}
\label{sec:rel}

The idea of using random projections to facilitate learning with high-dimensional data stems from the early work on random projections~\cite{DBLP:conf/focs/Dasgupta99} showing in particular that learning of high-dimensional mixtures of Gaussians can be simplified when first projecting the data into a randomly chosen subspace of low dimension (this is a consequence of the curse of dimensionality and the fact that high-dimensional data often has low intrinsic dimensionality). This idea was subsequently successfully applied to both synthetic and real datasets~\cite{conf/uai/Dasgupta00,Bingham:2001:RPD:502512.502546}, and then adopted to a number of learning approaches such as random projection trees~\cite{dasgupta2008random}, kernel and feature-selection techniques~\cite{Blum:2005:RPM:2182373.2182377}, clustering~\cite{conf/icml/FernB03a}, privacy-preserving machine learning~\cite{Liu:2006:RPM:1105850.1105890,conf/alt/ChoromanskaCJM13}, learning with large databases~\cite{Achlioptas:2003:DRP:861182.861189}, sparse learning settings~\cite{pingli2006very}, and more recently - deep learning (see~\cite{ICML2011Saxe_551} for convenient review of such approaches).
Using linear projections with completely random Gaussian weights, instead of learned ones, was recently studied from both theoretical and practical point of view in \cite{giryes}, but that work did not consider
structured matrices which is a central point of our interest since structured matrices can be stored much more efficiently. Beyond applying methods that use random Gaussian matrix projections~\cite{DBLP:conf/focs/Dasgupta99,conf/uai/Dasgupta00,giryes} and 
random binary matrix projections~\cite{Achlioptas:2003:DRP:861182.861189}, it is also possible to construct deterministic projections that preserve angles and distances~\cite{JXHC09}. In some sense these methods use structured
matrices as well, yet they do not have the same projection efficiency
of circulant matrices and projections explored in this article. Our hybrid approach, where a fixed ``budget of randomness'' is distributed across the entire matrix in the structured way enables us to take advantage of both: the ability of completely random projection to preserve information and the compactness that comes from the highly-organized internal structure of the linear mapping.

This work studies the paradigm of constructing a binary embedding for data compression, where hashes are obtained by applying linear projections to the data followed by the non-linear (sign function) mappings. The point-wise nonlinearity was not considered in many previous works on structured matrices~\cite{haupt2010toeplitz,journals/corr/abs-1010-1847,CPA:CPA21504,yap.11b} (moreover note that these works also consider the set of structured matrices which is a strict subset of the class of matrices considered here). 
Designing binary embeddings for high dimensional data with low distortion is addressed in many recent 
works~\cite{weiss, ssh, gong2, gong, aqbe, felix, constantine, NIPS2009_3749,Salakhutdinov:2009:SH:1558385.1558446}. 
In the context of our work, one of the recent articles~\cite{constantine} is especially important 
since the authors introduce the pipeline of constructing hashes with the use of structured matrices in the linear step, instead of completely random ones. They prove several theoretical results regarding the quality of the produced hash, and extend some previous theoretical results~\cite{DBLP:journals/corr/abs-1104-3160,plan}. Their pipeline is more complicated than ours, i.e. they first apply Hadamard transformation and then a sequence of partial Toeplitz Gaussian matrices.
Some general results (unbiasedness of the angular distance estimator) were also known for short hashing pipelines involving circulant matrices (\cite{felix}). 
These works do not provide guarantees regarding concentration of the estimator around its mean, which is crucial for all practical applications.
Our results for general structured matrices, which include circulant Gaussian matrices and a larger class of Toeplitz Gaussian matrices as special subclasses, 
provide such concentration guarantees, and thus establish a solid mathematical foundation for using various types of structured matrices in binary embeddings. 
In contrast to \cite{constantine}, we present our theoretical results for simpler hashing models (our hashing mechanism is explained in Section~\ref{sec:has_mech} 
and consists of two very simple steps that we call preprocessing step and the actual hashing step, where the latter consists of pseudo-random projection followed 
by the nonlinear mapping). In \cite{aditya} theoretical guarantees regarding bounding the variance of the angle estimator in the circulant setting were presented.
Strong concentration results regarding several structured matrices were given in \cite{choromanski_sindhwani, triple_spin}, following our work.

In the context of deep learning, using random network parametrization, where certain layers have random and untrained weights, often accelerates training. Introducing randomness to networks was explored for various architectures, in example feedforward networks~\cite{Huang2006489}, convolutional networks~\cite{conf/iccv/JarrettKRL09,ICML2011Saxe_551}, and recurrent networks~\cite{Jaeger04harnessingnonlinearity,citeulike:13373321,Boedecker2009}. We also refer the reader to~\cite{3447}, where the authors describe how neural systems cope with the challenge of processing data in high dimensions and discuss random projections. Hashing in neural networks that we consider in this paper is a new direction of research. Very recently (see: \cite{tyree}) it was empirically showed that hashing in neural nets may achieve drastic reduction in model sizes with no significant loss of the quality, by heavily exploiting the phenomenon of redundancies in neural nets. 
HashedNets introduced in \cite{tyree} do not give any theoretical guarantees regarding the quality of the proposed hashing. Our work aims to touch both grounds. We experimentally show the plausibility of the approach, but also explain theoretically why the hashing we consider compresses important information 
about the data that suffices for accurate classification. 
Dimensionality reduction techniques were used also to approximately preserve certain metrics defined on graph objects (\cite{shaw_jebara}).
Structured hashing was applied also in \cite{yann_lecun}, but in a very different context than ours. 

\section{Hashing mechanism}
\label{sec:has_mech}

In this section we explain our hashing mechanism for dimensionality reduction that we next analyze.

\subsection{Structured matrices}

We first introduce the aforementioned family of structured matrices, that we call: $\Psi$-regular matrices $\mathcal{P}$. This is the key ingredient of the method.

\begin{definition}
\label{circ-def}
$\mathcal{M}$ is a circulant Gaussian matrix if its first row is a sequence of independent Gaussian random variables taken from the distribution $\mathcal{N}(0,1)$ and next rows are obtained from the previous ones by either only one-left shifts or only one-right shifts. 
\end{definition}
\begin{definition}
\label{Toep-def}
$\mathcal{M}$ is a Toeplitz Gaussian matrix if each of its descending diagonals from left to right is of the form $(g,...,g)$ for some $g \sim \mathcal{N}(0,1)$ and different descending diagonals are independent. 
\end{definition}

\begin{remark} 
The circulant Gaussian matrix with right shifts is a special type of the Toeplitz Gaussian matrix.
\end{remark}
\vspace{-0.02in}
Assume that $k$ is the size of the hash and $n$ is the dimensionality of the data.

\begin{definition}
Let $t$ be the size of the pool of independent random Gaussian variables $\{g_{1},...,g_{t}\}$,
where each $g_{i} \sim \mathcal{N}(0,1)$. Assume that $k \leq n \leq t \leq kn$. We
take $\Psi$ to be a natural number, i.e. $\Psi \in \mathbb{N}$. $\mathcal{P}$ is $\Psi$-regular random matrix if it has the following form
\vspace{-0.02in}
\begin{equation*}
\left( \begin{array}{ccccc}
\sum_{l \in \mathcal{S}_{1,1}}g_{l}     &  ... & \sum_{l \in \mathcal{S}_{1,j}}g_{l} & ... & \sum_{l \in \mathcal{S}_{1,n}}g_{l}\\
... & ... & ... & ... & ... \\
\sum_{l \in \mathcal{S}_{i,1}}g_{l} & ... & \sum_{l \in \mathcal{S}_{i,j}}g_{l} & ... & \sum_{l \in \mathcal{S}_{i,n}}g_{l}\\
... & ... & ... & ... & ... \\
\sum_{l \in \mathcal{S}_{k,1}}g_{l} & ... & \sum_{l \in \mathcal{S}_{k,j}}g_{l} & ... & \sum_{l \in \mathcal{S}_{k,n}}g_{l}
\end{array} \right)
\end{equation*}

where $S_{i,j} \subseteq \{1,...,t\}$ for $i \in \{1,...,k\}$, $j \in \{1,...,n\}$, $|S_{i,1}|=...=|S_{i,n}|$ for $i=1,...,k$, $S_{i,j_{1}} \cap S_{i,j_{2}} = \emptyset$
for $i \in \{1,...,k\}$, $j_{1},j_{2} \in \{1,...,n\}$, $j_{1} \neq j_{2}$, and furthermore the following holds: 
for any two different rows $\mathcal{R}_{1}, \mathcal{R}_{2}$ of $\mathcal{P}$ 
the number of random variables $g_{l}$, where $l \in \{1,...,t\}$, such that $g_{l}$ is in the intersection of some column with both $\mathcal{R}_{1}$
and $\mathcal{R}_{2}$ is at most $\Psi$.
\end{definition}

\begin{remark} 
Circulant Gaussian matrices and Toeplitz Gaussian matrices are special cases of the $0$-regular matrices. Toeplitz Gaussian matrix is $0$-regular, where subsets $\mathcal{S}_{i,j}$ are singletons.
\end{remark}

In the experimental section of this paper we consider six different kinds of structured matrices, which are examples of general structured matrices covered by our theoretical analysis. 
Those are:
\vspace{-0.15in}
\begin{itemize}
\setlength\itemsep{-0.1em}
\item \textit{Circulant} (see: Definition \ref{circ-def}),
\item \textit{BinCirc} - a matrix, where the first row is partitioned into consecutive equal-length blocks of elements and each row is obtained by the right shift of the blocks from the first row,
\item \textit{HalfShift} - a matrix, where next row is obtained from the previous one by swapping its halves and then performing right shift by one,
\item \textit{VerHorShift} - a matrix that is obtained by the following two phase-procedure: first each row is obtained from the previous one by the right shift of a fixed length and then in the obtained matrix each column is shifted up by a fixed number of elements,
\item \textit{BinPerm} - a matrix, where the first row is partitioned into consecutive equal-length blocks of elements and each row is obtained as a random permutation of the blocks from the first  row,
\item \textit{Toeplitz} (see: Definition~\ref{Toep-def}).
\end{itemize} 

\begin{remark}
Computing hashes for structured matrices: Toeplitz, BinCirc, HalfShift, and VerHorShift can be done faster than in time $\mathcal{O}(nk)$ (e.g. for Toeplitz one can use FFT to reduce computations to $\mathcal{O}(n\log k)$). Thus our structured approach leads to speed-ups, storage compression (since many structured matrices covered by our theoretical model can be stored in linear space) and reduction in randomness usage.
The goal of this paper is not to analyze in details fast matrix-vector product algorithms
since that requires a separate paper. We however point out that well-known fast matrix-vector product algorithms are some of the key advantages of our structured approach.

\end{remark}

\subsection{Hashing methods}

Let $\phi$ be a function satisfying $\lim_{x \rightarrow \infty} \phi(x) = 1$ and $\lim_{x \rightarrow -\infty} \phi(x) = -1$.
We will consider two hashing methods, both of which consist of what we refer to as a \textit{preprocessing} step followed by the actual \textit{hashing} step, where the latter consists of pseudo-random projection followed by nonlinear (sign function) mapping. The first mechanism, that we call \textit{extended $\Psi$-regular hashing}, applies first random diagonal matrix $\mathcal{R}$ to the data point $x$,
then the $L_{2}$-normalized Hadamard matrix $\mathcal{H}$, next another random diagonal matrix $\mathcal{D}$, then the $\Psi$-regular projection matrix $\mathcal{P}_{\Psi}$ and finally function $\phi$ (the latter one applied point-wise).
The overall scheme is presented below:
\begin{equation}
\underbrace{x \xrightarrow {\mathcal{R}} x_{\mathcal{R}} \xrightarrow {\mathcal{H}} x_{\mathcal{H}} \xrightarrow {\mathcal{D}} x_{\mathcal{D}}}_{\text{preprocessing}} \underbrace{\xrightarrow {\mathcal{P}_{\Psi}} x_{\mathcal{P}_{\Psi}} \xrightarrow {\phi} h(x)}_{\text{hashing}} \in \mathbb{R}^{k}.
\end{equation}
The diagonal entries of matrices $\mathcal{R}$ and $\mathcal{D}$ are chosen independently from the binary set $\{-1,1\}$,
each value being chosen with probability $\frac{1}{2}$.
We also propose a shorter pipeline, that we call \textit{short $\Psi$-regular hashing}, where we avoid applying first random matrix $\mathcal{R}$ and the Hadamard matrix $\mathcal{H}$, i.e. the overall pipeline is of the form:
\begin{equation}
\underbrace{x  \xrightarrow {\mathcal{D}} x_{\mathcal{D}}}_{\text{preprocessing}} \underbrace{\xrightarrow {\mathcal{P}_{\Psi}} x_{\mathcal{P}_{\Psi}} \xrightarrow {\phi} h(x)}_{\text{hashing}} \in \mathbb{R}^{k}.
\end{equation}
The goal is to compute good approximation of the angular distance between given vectors $p,r$, given their 
compact hashed versions: $h(p), h(r)$. 
To achieve this goal we consider the $L_{1}$-distances in the $k$-dimensional space of hashes.
Let $\theta_{p,r}$ denote the angle between vectors $p$ and $r$. We define the \textit{normalized approximate angle between $p$ and $r$} as:  
\begin{equation}
\tilde{\theta}_{p,r}^{n} = \frac{1}{2k}\|h(p)-h(r)\|_{1}
\end{equation} 
In the next section we show that the normalized approximate angle between vectors $p$ and $r$ leads to a very precise estimation of the actual angle for $\phi(x)=sign(x)$ if the chosen parameter $\Psi$ is not too large. Furthermore, we show an intriguing connection between theoretical guarantees regarding the quality of the produced hash and the chromatic number of some specific undirected graph encoding the structure of $\mathcal{P}$. For many of the structured matrices under consideration this graph is induced by an algebraic group operation defining the structure of $\mathcal{P}$ (for instance, for the circular matrix the group is a single shift and the underlying graph is a collection of pairwise disjoint cycles, thus its chromatic number is at most $3$).

\section{Theoretical results}
\label{sec:the}
\subsection{Unbiasedness of the estimator}
We are ready to provide theoretical guarantees regarding the quality of the produced hash. Our guarantees will be given for a \textit{sign} function, i.e for $\phi$ defined as: $\phi(x) = 1$ for $x \geq 0$, $\phi(x) = -1$ for $x < 0$. Using this nonlinearity will be important to preserve approximate information about the angle between vectors, while filtering out the information about their lengths. We first show that $\tilde{\theta}_{p,r}^{n}$ is an unbiased estimator of $\frac{\theta_{p,r}}{\pi}$, i.e. $E(\tilde{\theta}_{p,r}^{n}) = \frac{\theta_{p,r}}{\pi}$.

\begin{lemma}
\label{mean_lemma}
Let $\mathcal{M}$ be a $\Psi$-regular hashing model (either extended or a short one) and $\|p\|_{2}=\|r\|_{2}=1$. 
Then $\tilde{\theta}_{p,r}^{n}$ is an unbiased estimator of $\frac{\theta_{p,r}}{\pi}$, i.e. $E(\tilde{\theta}_{p,r}^{n}) = \frac{\theta_{p,r}}{\pi}.$
\end{lemma}

Let us give a short sketch of the proof first. Note that the value of the particular entry in the constructed hash depends only on the sign of the dot product between the corresponding Gaussian vector representing the row of the $\Psi$-regular matrix and the given vector.  Fix two vectors: $p$ and $r$ with angular distance $\theta$. Note that considered dot products (and thus also their signs) are preserved when instead of taking the Gaussian vector representing the row one takes its projection onto a linear space spanned by $p$ and $r$.
The Hamming distance between hashes of $p$ and $r$ is built up by these entries for which one dot product is negative and the other one is positive. One can note that this happens if the projected vector is inside a $2$-dimensional cone covering angle $2\theta$. 
The last observation that completes the proof is that the projection of the Gaussian vector is isotropic 
(since it is also Gaussian), thus the probability that the two dot products will have different signs is $\frac{\theta}{\pi}$
(also see: \cite{charikar}).
\begin{proof}

Note first that the $i$th row, call it $g^{i}$, of the matrix $\mathcal{P}$ is a $n$-dimensional Gaussian vector with mean $0$ and where each element has variance $\sigma_{i}^{2}$ for $\sigma_{i}^{2}=|\mathcal{S}_{i,1}|=...=|\mathcal{S}_{i,n}|$ ($i=1,...,k$).
Thus, after applying matrix $\mathcal{D}$ the new vector $g^{i}_{\mathcal{D}}$ is still Gaussian and of the same distribution.
Let us consider first the short $\Psi$-regular hashing model.
Fix some vectors $p,r$ (without loss of generality we may assume that they are not collinear). Let $H_{p,r}$, shortly called by us $H$, be the $2$-dimensional hyperplane spanned by $\{p,r\}$. Denote by $g^{i}_{\mathcal{D},H}$ the projection of 
$g^{i}_{\mathcal{D}}$ into $H$ and by $g^{i}_{\mathcal{D},H,\perp}$ the line in $H$ perpendicular to $g^{i}_{\mathcal{D},H}$. Let $\phi$ be a \textit{sign} function. Note that the contribution to the $L_{1}$-sum $\|h(p)-h(r)\|_{1}$ comes from those $g^{i}$ for which
$g^{i}_{\mathcal{D},H,\perp}$ divides an angle between $p$ and $r$ (see: Figure~\ref{fig:one}), i.e. from those $g^{i}$ for which $g^{i}_{\mathcal{D},H}$
is inside the union $\mathcal{U}_{p,r}$ of two $2$-dimensional cones bounded by two lines in $H$ perpendicular to $p$ and $r$
respectively. If the angle is not divided (see: Figure~\ref{fig:two}) then the two corresponding entries in the hash have the same value and thus do not contribute to the overall distance between hashes.

Observe that, from what we have just said, we can conclude that $\tilde{\theta}_{p,r}^{n} = \frac{X_{1} + ... + X_{k}}{k}$, where:
\begin{equation}
X_{i} =
\left\{
	\begin{array}{ll}
		1  & \mbox{if }  g^{i}_{\mathcal{D},H} \in \mathcal{U}_{p,r}, \\
		0 & \mbox{otherwise.} 
	\end{array}
\right.
\end{equation}

\begin{figure}[t]
\centering
\includegraphics[width = 2.8in]{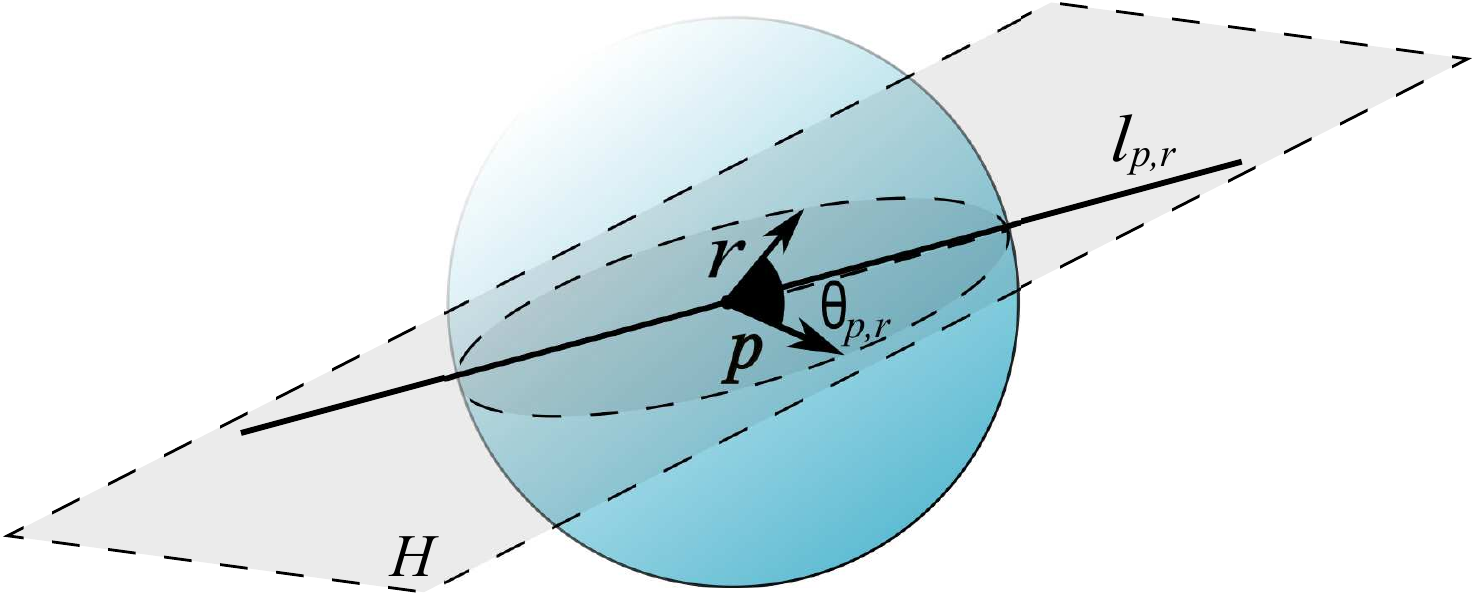}
\vspace{-0.05in}
\caption{Two vectors: $p$, $r$ spanning two-dimensional hyperplane $H$ and with the angular distance $\theta_{p,r}$ between them. We have: $l_{p,r} = g^{i}_{\mathcal{D},H,\perp}$.
Line $l_{p,r}$ is dividing $\theta_{p,r}$ and thus $g^{i}$ contributes to $\|h(p)-h(r)\|_{1}$.}
\label{fig:one}
\vspace{-0.15in}
\end{figure}
\begin{figure}[t]
\includegraphics[width = 2.3in]{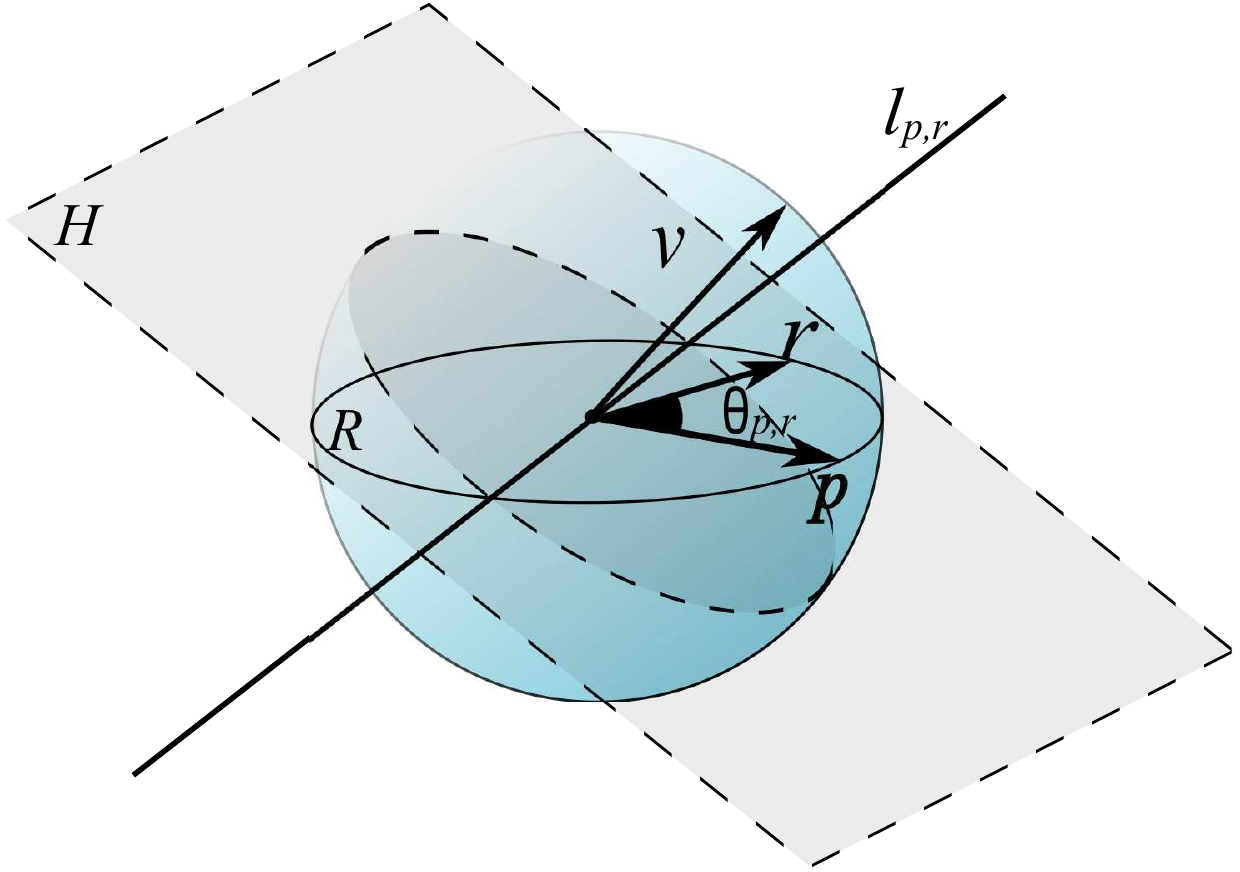}
\vspace{-0.05in}
\caption{Similar setting to the one presented on Figure \ref{fig:one}. Vector $v$ represents $L_{2}$-normalized version of $g^{i}$ and is perpendicular to the two-dimensional plane $R$.
The intersection $R \cap H$ of  that plane with the $2$-dimensional plane $H$ spanned by $p,r$ is 
a line $l_{p,r}$ that this time is outside $\mathcal{U}_{p,r}$. Thus $g^{i}$ does not contribute to 
$\|h(p)-h(r)\|_{1}$.}
\label{fig:two}
\centering
\vspace{-0.1in}
\end{figure}

Now it suffices to note that vector $g^{i}_{\mathcal{D},H}$ is a $2$-dimensional Gaussian vector and thus its direction is uniformly distributed over all directions. Thus each $X_{i}$ is nonzero with probability exactly $\frac{\theta_{p,r}}{\pi}$ and the theorem follows.
For the extended $\Psi$-regular hashing model the analysis is very similar.
The only difference is that data is preprocessed by applying $\mathcal{H}\mathcal{R}$ linear mapping first. 
Both $\mathcal{H}$ and $\mathcal{R}$ are orthogonal matrices though, thus their product is also an orthogonal matrix. Since orthogonal matrices do not change angular distance, the former analysis can be applied again and yields the proof.
\end{proof}

We next focus on the concentration of the random variable $\tilde{\theta}_{p,r}^{n}$ around its mean $\frac{\theta_{p,r}}{\pi}$. We prove strong exponential concentration results for the extended $\Psi$-regular hashing method. Interestingly, the application of the Hadamard mechanism is not necessary and it is possible to get concentration results, yet weaker than in the former case, also for short $\Psi$-regular hashing.

\subsection{The $\mathcal{P}$-chromatic number}

The highly well organized structure of the projection matrix $\mathcal{P}$ gives rise
to the underlying undirected graph that encodes dependencies between different entries of $\mathcal{P}$.
More formally, let us fix two rows of $\mathcal{P}$ of indices $1 \leq k_{1} < k_{2} \leq k$ respectively.
We define a graph $\mathcal{G}_{\mathcal{P}}(k_{1},k_{2})$ as follows: 
\vspace{-0.1in}
\begin{itemize}
\setlength\itemsep{-0.1em}
\item $V(\mathcal{G}_{\mathcal{P}}(k_{1},k_{2})) = \{\{j_{1},j_{2}\}: \exists l \in \{1,...,t\} s.t.  g_{l} \in \mathcal{S}_{k_{1},j_{1}} \cap \mathcal{S}_{k_{2},j_{2}}, j_{1} \neq j_{2}\}$,
\item there exists an edge between vertices $\{j_{1},j_{2}\}$ and $\{j_{3},j_{4}\}$ iff $\{j_{1},j_{2}\} \cap \{j_{3},j_{4}\} \neq \emptyset$.
\end{itemize}
\vspace{-0.1in}

The chromatic number $\chi(\mathcal{G})$ of the graph $\mathcal{G}$ is the minimal number of colors that can be used to color the vertices of the graph in such a way that no two adjacent vertices have the same color.

\begin{definition}
Let $\mathcal{P}$ be a $\Psi$-regular matrix. We define the $\mathcal{P}$-chromatic number $\chi(\mathcal{P})$ 
as:
$$\chi(\mathcal{P}) = \max_{1 \leq k_{1} < k_{2} \leq k} \chi(\mathcal{G}(k_{1},k_{2})).$$
\end{definition}

\begin{figure}[h]
\vspace{-0.3in}
\center
\includegraphics[width = 3.3in]{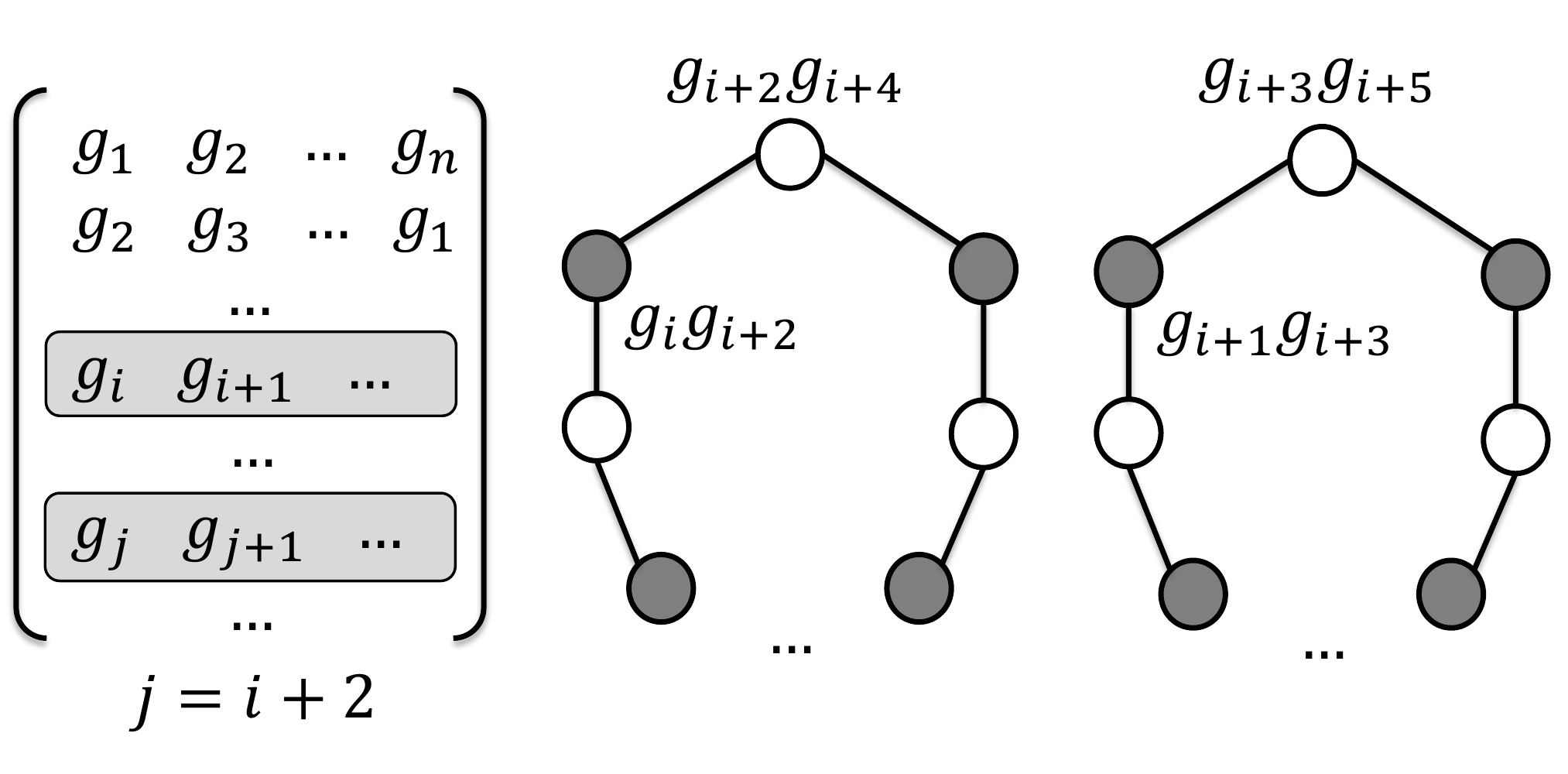}
\vspace{-0.35in}
\caption{\textbf{Left:} matrix $\mathcal{P}$ with two highlighted rows of indices: $k_{1}=i$ and $k_{2}=j$ respectively, where $j=i+2$. \textbf{Right:} corresponding graph that consists of two cycles. If each cycle is even then this graph is $2$-colorable, as indicated on the picture. Thus we have: $\chi(\mathcal{G}_{\mathcal{P}}(k_{1},k_{2}))=2$.}
\label{fig:fig2}
\vspace{-0.1in}
\label{fig:chromatic}
\end{figure}

The graph associated with each structured matrix that we have just described enables us to encode dependencies between entries of the structured matrix in the compact form
and gives us quantitative ways to efficiently measure these dependencies by analyzing several core parameters of this graph such as its chromatic number.
More dependencies that usually lead to more structured form mean more edges in the associated graph and often lead to higher chromatic number. On the other hand, fewer dependencies
produce graphs with much lower chromatic number (see Figure \ref{fig:chromatic}, where the graph associated with the circulant matrix is a collection of vertex disjoint cycles and has  chromatic number $3$ if it contains an odd length cycle and $2$ otherwise).

\subsection{Concentration inequalities for structured hashing with \textit{sign} function}

We present now our main theoretical results. The proofs are deferred to the Supplementary material. 
We focus on the concentration results regarding produced hashes that are crucial for practical applications of the proposed scheme.

We first start with the short description of the methods used and then rigorously formulate all the results.
If all the rows of the projection matrix are independent then standard concentration inequalities can be used. This is however not the case
in our setting since the matrix is structured. We still want to say that any two rows are ``close'' to independent Gaussian vectors and that will give us
bounds regarding the variance of the distance between the hashes (in general, we observe that any system of $k$ rows is ``close'' to the system of $k$ independent Gaussian vectors
and get bounds involving $k$th moments). 
We proceed as follows:
\vspace{-0.1in}
\begin{itemize}
\setlength\itemsep{-0.1em}
\item We take two rows and project them onto the linear space spanned by given vectors: $p$ and $r$. 
\item We consider the four coordinates obtained in this way (two for each vector). They are obviously Gaussian, but what is crucial, they are ``almost independent''.
\item  The latter observation is implied by the fact
that these are the coordinates of the projection of a fixed Gaussian vector onto ``almost orthogonal' directions'.
\item We use the property of the Gaussian vector that its projections onto orthogonal directions are independent.
\item To prove that directions considered in our setting are close to orthogonal with high probability, we compute their dot product. This is the place where the structure of the matrix, the chromatic number of the underlying graph and the fact that in our hashing scheme we use random diagonal matrices come into action. We decompose each dot product into roughly speaking $\chi$ components ($\chi$ is the chromatic number), such that each component is a sum of independent random variables with mean $0$. Now we can use standard concentration inequalities to get tight concentration results.
\item The Hadamard matrix used in the extended model preprocesses input vectors to distribute their mass uniformly over all the coordinates, while not changing $L_{2}$ distances (it is a unitary matrix). Balanced vectors lead to much stronger concentration results.
\end{itemize}
\vspace{-0.1in}

Now we are ready to rigorously state our results. 
By $poly(x)$ we denote a function $x^{r}$ for some $r > 0$.
The following theorems guarantee strong concentration of
${\tilde \theta}^n_{p,r}$ around its mean and therefore justify
theoretically the effectiveness of the structured hashing method. 

Let us consider first the extended $\Psi$-regular hashing model. 

\begin{theorem}
\label{ext_technical_theorem}
Consider extended $\Psi$-regular hashing model $\mathcal{M}$ with $t$ independent Gaussian random variables: $g_{1},...,g_{t}$, each of distribution $\mathcal{N}(0,1)$.
Let $N$ be the size of the dataset $\mathcal{D}$. Denote by $k$ the size of the hash and by $n$ the dimensionality of the data. Let $f(n)$ be an arbitrary positive function. Let $\theta_{p,r}$ be the angular distance between vectors $p,r \in \mathcal{D}$.
Then for $a=o_{n}(1)$, $\epsilon>0$, $t \geq n$ and $n$ large enough:
\vspace{-0.07in}
\begin{eqnarray*}
&&\!\!\!\!\!\!\!\!\!\!\!\!\mathbb{P}\left(\forall_{p,r \in\mathcal{D}}\left|\tilde{\theta}^{n}_{p,r} - \frac{\theta_{p,r}}{\pi}\right| \leq \epsilon\right) \geq\\ &&\!\!\!\!\!\!\!\!\!\!\!\!\left[1-4{N \choose 2}e^{-\frac{f^{2}(n)}{2}}-4\chi(\mathcal{P}){k \choose 2}e^{-\frac{2a^{2}t}{f^{4}(t)}}\right](1-\Lambda),
\end{eqnarray*}
where $\Lambda = \frac{1}{\pi} \sum_{j = \lfloor \frac{\epsilon k}{2} \rfloor}^{k}\frac{1}{\sqrt{j}}(\frac{ke}{j})^{j}\mu^{j}(1-\mu)^{k-j}+2e^{-\frac{\epsilon^{2}k}{2}}$ and $\mu=\frac{8(\sqrt{a}\chi(\mathcal{P}) + \Psi\frac{f^{2}(n)}{\sqrt{n}})}{\theta_{p,r}}$.
\end{theorem}

Note how the upper bound on the probability of failure $\mathbb{P}_{\epsilon}$ depends on the $\mathcal{P}$-chromatic number. The theorem above guarantees strong concentration of $\tilde{\theta}^{n}_{p,r}$ around its mean and therefore justifies theoretically the effectiveness of the structured hashing method. It becomes more clear below.

As a corollary, we obtain the following result:

\begin{corollary}
\label{ext_theorem}
Consider extended $\Psi$-regular hashing model $\mathcal{M}$. Assume that the projection matrix $\mathcal{P}$ is Toeplitz Gaussian.
Let $N, n, k$ be as above and denote by $\theta_{p,r}$ be the angular distance between vectors $p,r \in \mathcal{D}$.
Then the following is true for $n$ large enough: 
\vspace{-0.05in}
\begin{eqnarray*}
&&\!\!\!\!\!\!\!\!\!\!\!\!\mathbb{P}\left(\forall_{p,r \in\mathcal{D}}\left|\tilde{\theta}^{n}_{p,r} - \frac{\theta_{p,r}}{\pi}\right| \leq k^{-\frac{1}{3}}\right) 
\geq\\ &&\!\!\!\!\!\!\!\!\!\!\!\!\left[1-O\left(\frac{N^{2}}{e^{poly(n)}}+
k^{2}e^{-n^{\frac{3}{10}}}\right)\right]\left(1-3e^{-\frac{k^{\frac{1}{3}}}{2}}\right).
\end{eqnarray*}
\end{corollary}


\begin{figure}[h!]
\vspace{-0.2in}
\centering
\hspace{-0.1in}\includegraphics[width = 1.7in]{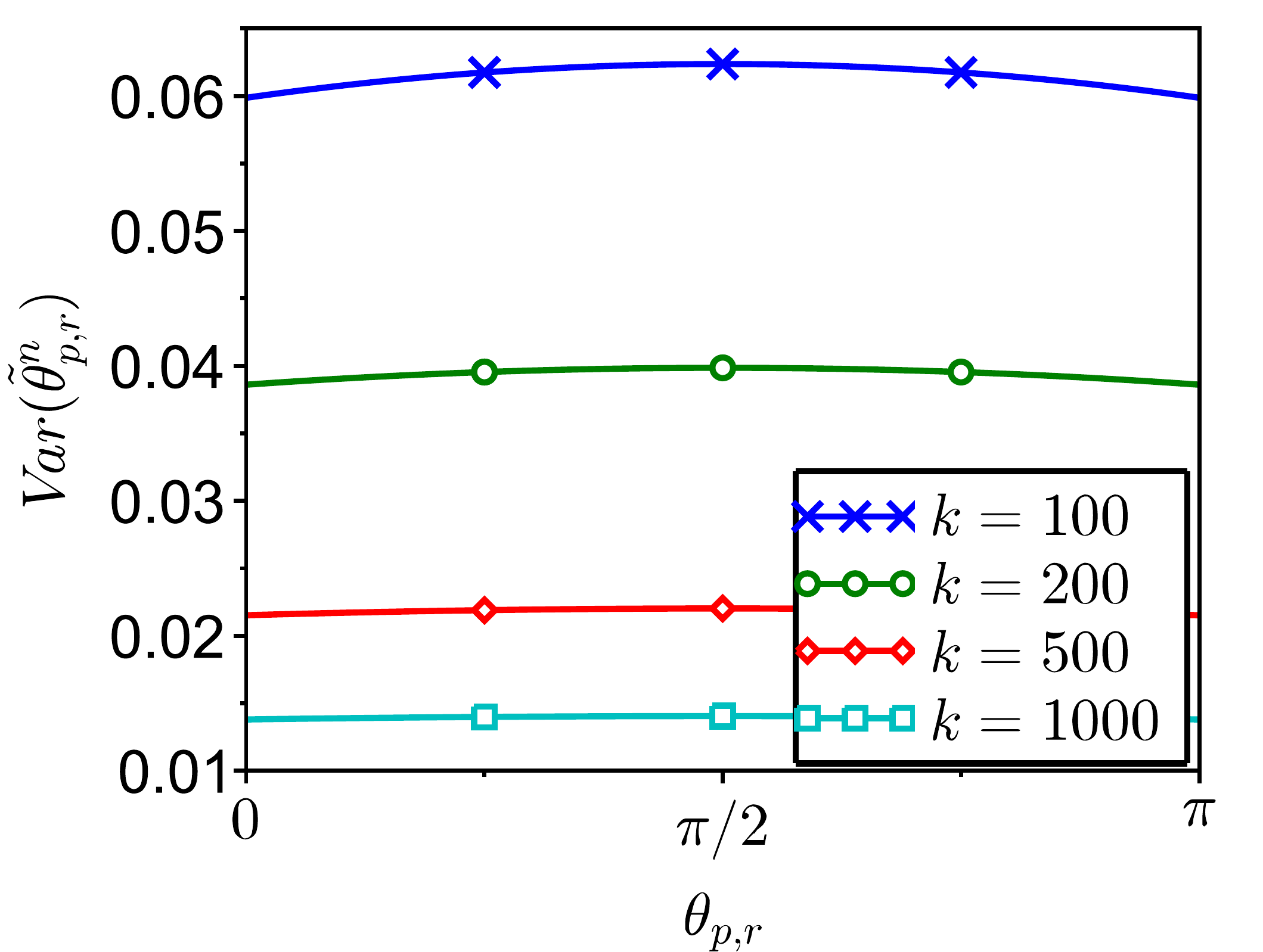}
\hspace{-0.1in}\includegraphics[width = 1.7in]{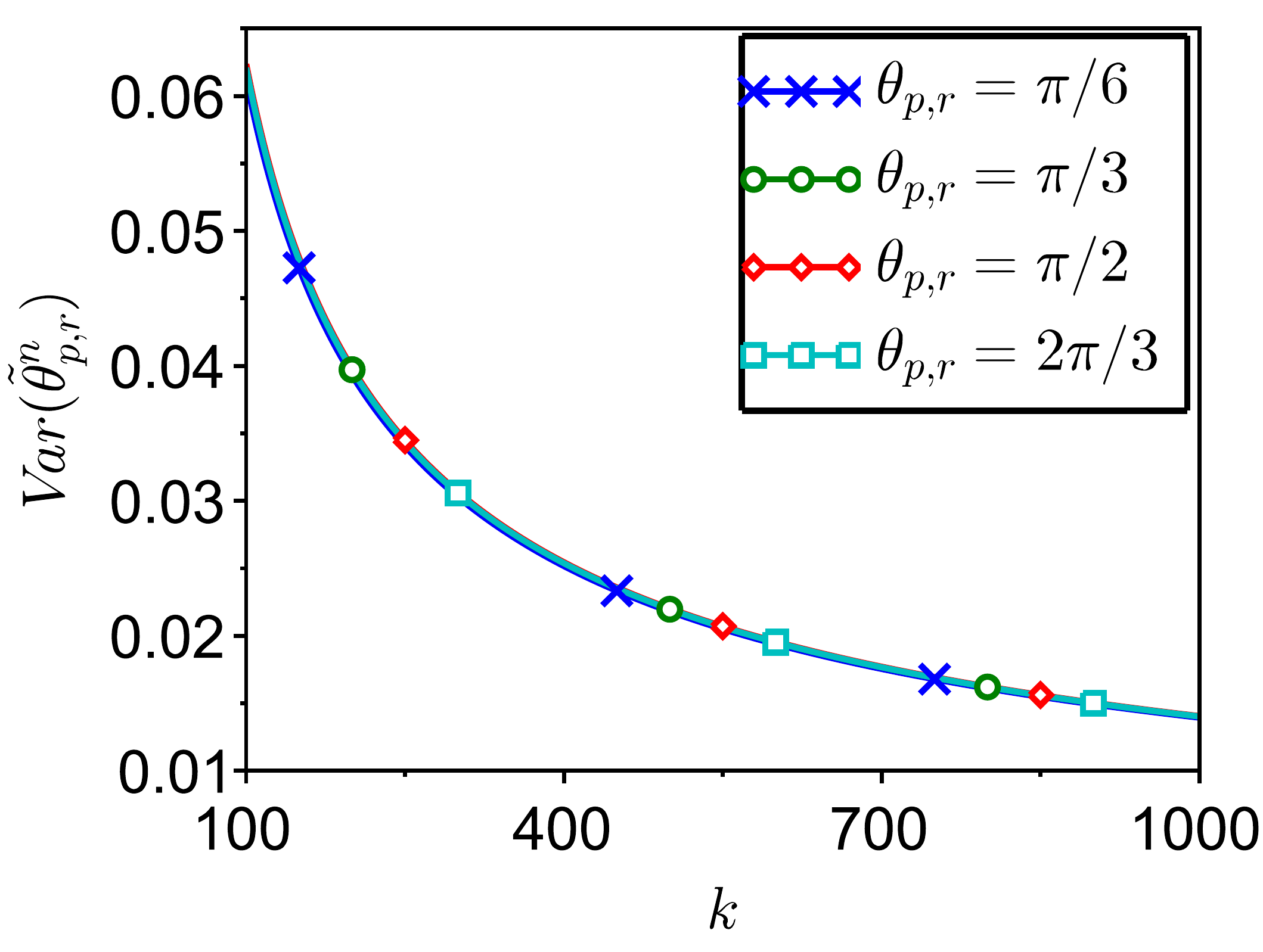}
\vspace{-0.25in}
\caption{The dependence of the upper bound on the variance of the normalized approximate angle $\tilde{\theta}^{n}_{p,r}$ 
on \textbf{(left:)} an angle when the size of the hash $k$ is fixed (the upper bound scales as $\frac{1}{k}$
and is almost independent of $\theta_{p,r}$), \textbf{(right:)} the size of the hash $k$ when the true angular distance $\theta_{p,r}$ is fixed (the upper bound converges to $0$ as $k \rightarrow \infty$).}
\label{fig:theory1}
\vspace{-0.1in}
\end{figure}

Corollary \ref{ext_theorem} follows from Theorem \ref{ext_technical_theorem} by taking:
$a=n^{-\frac{1}{3}}$, $\epsilon = k^{-\frac{1}{3}}$, $f(n)=n^{p}$ for small enough constant $p>0$, noticing that every Toeplitz Gaussian matrix is $0$-regular and the corresponding $\mathcal{P}$-chromatic number $\chi(\mathcal{P})$ is at most $3$. 

Term $O\left(\frac{N^{2}}{e^{poly(n)}}\right)$ is related to the balancedness property. To clarify, the goal of multiplying by $\mathcal{H}\mathcal{R}$ in the preprocessing step is to make each input vector balanced, or in other words to spread out the mass of the vector across all the dimensions in approximately uniform way. This property is required to obtain theoretical results (also note it was unnecessary in the unstructured setting) and does not depend on the number of projected dimensions.

Let us consider now the short $\Psi$-regular hashing model. The theorem presented below is an application of the 
Chebyshev's inequality preceded by the careful analysis of the variance of $\tilde{\theta}^{n}_{p,r}$.

\begin{theorem}
\label{short_theorem}
Consider short $\Psi$-regular hashing model $\mathcal{M}$, where $\mathcal{P}$ is a Toeplitz Gaussian matrix.
Denote by $k$ the size of the hash.
Let $\theta_{p,r}$ be the angular distance between vectors $p,r \in \mathcal{D}$, where $\mathcal{D}$ is the dataset.
Then the following is true
\begin{equation}
\forall_{p,r \in\mathcal{D}}Var(\tilde{\theta}^{n}_{p,r}) \leq \frac{1}{k}\frac{\theta_{p,r}(\pi-\theta_{p,r})}{\pi^{2}} + (\frac{\log(k)}{k^{2}})^{\frac{1}{3}},
\end{equation}
and thus for any $c>0$ and $p,r \in \mathcal{D}$:
\vspace{-0.1in}
$$\mathbb{P}\left(\left|\tilde{\theta}^{n}_{p,r} - \frac{\theta_{p,r}}{\pi}\right| \geq c \left(\frac{\sqrt{\log(k)}}{k}\right)^{\frac{1}{3}}\right) = O\left(\frac{1}{c^{2}}\right).$$
\end{theorem}
\vspace{-0.1in}

Figure~\ref{fig:theory1} shows the dependence of the upper bound on the variance of the normalized approximate angle $\tilde{\theta}^{n}_{p,r}$ on resp. the true angular distance $\theta_{p,r}$ and the size of the hash $k$ when resp. $k$ and $\theta_{p,r}$ are fixed.

Rate $k^{-\frac{1}{3}}$ that appears in the theoretical results we obtained and the non-linear with $k$ variance decay of Figure~\ref{fig:theory1} (right) is a consequence of the structured setting, where the quality of the nonlinear embedding is affected by the existence of dependencies between entries of the structured matrix.

\section{Numerical experiments}
\label{sec:ne}

In this section we demonstrate that all considered structured matrices achieve reasonable performance in comparison to fully random matrices. Specifically we show: i) the dependence of the performance on the size of the hash and the reduction factor $\frac{n}{k}$ for different structured matrices and ii) the performance of different structured matrices when used with neural networks and $1$-NN classifier. Experiments confirm our novel theoretical results.

\vspace{-0.1in}
\begin{figure}[h!]
\centering
\begin{tikzpicture}[shorten >=1pt,->,draw=black!50, node distance=\layersep]
    \tikzstyle{every pin edge}=[<-,shorten <=1pt]
    \tikzstyle{neuron}=[circle,fill=black!27,minimum size=20pt,inner sep=0pt]
    \tikzstyle{input neuron}=[neuron, fill=green!50];
    \tikzstyle{output neuron}=[neuron, fill=red!50];
    \tikzstyle{hidden neuron}=[neuron, fill=blue!50];
     \tikzstyle{transparent neuron}=[neuron, fill=blue!0];
    \tikzstyle{annot} = [text width=4em, text centered]

        \node[input neuron, pin=left:$x_1$] (I-1) at (0,-0.5) {};
        \node[input neuron, pin=left:$x_2$] (I-2) at (0,-2) {};
        \node[transparent neuron, pin=left:$\dots$] (I-3) at (0,-3) {$\dots$};	
        \node[input neuron, pin=left:$x_n$] (I-4) at (0,-4.5) {};

    \foreach \name / \y in {1,...,2}
        \path[yshift=0cm]
            node[hidden neuron] (H1-\name) at (\layersep,-\y cm) {\textbf{\y}};
        \path[yshift=0cm]
            node[transparent neuron] (H1-3) at (\layersep,-3 cm) {$\dots$};
        \path[yshift=0cm]
            node[hidden neuron] (H1-4) at (\layersep,-4 cm) {$\bf{k}$};

    \foreach \name / \y in {1,...,2}
        \path[yshift=0cm]
            node[hidden neuron] (H2-\name) at (\layersepp,-\y cm) {};
        \path[yshift=0cm]
            node[transparent neuron] (H2-3) at (\layersepp,-3 cm) {$\dots$};
        \path[yshift=0cm]
	      node[hidden neuron] (H2-4) at (\layersepp,-4 cm) {};




        \path[yshift=0cm]
            node[output neuron] (O-1) at (\layersepppp,-1.2 cm) {$y_1$};
        \path[yshift=0cm]
            node[output neuron] (O-2) at (\layersepppp,-2 cm) {$y_2$};
         \path[yshift=0cm]
            node[transparent neuron] (O-3) at (\layersepppp,-3 cm) {$\dots$};
        \path[yshift=0cm]
            node[output neuron] (O-4) at (\layersepppp,-3.8 cm) {$y_s$};

  \foreach \source in {1,...,4}
      \foreach \dest in {2,...,4}
          \path (I-\source) edge [ultra thick,red] (H1-\dest);

  \foreach \source in {1,...,1}
      \foreach \dest in {1,...,1}
	\path (I-\source) edge [ultra thick,red] node [above] {\textcolor{red}{$\bf \mathcal{P}_{\text{
}}\mathcal{D}$}} (H1-\dest);

  \foreach \source in {2,...,4}
      \foreach \dest in {1,...,1}
	\path (I-\source) edge [ultra thick,red] (H1-\dest);

  \foreach \source in {1,...,4}
      \foreach \dest in {2,...,4}
          \path (H1-\source) edge [ultra thick] (H2-\dest);

  \foreach \source in {1,...,1}
      \foreach \dest in {1,...,1}
	\path (H1-\source) edge [ultra thick] (H2-\dest);

  \foreach \source in {2,...,4}
      \foreach \dest in {1,...,1}
	\path (H1-\source) edge [ultra thick] (H2-\dest);


 \foreach \source in {1,...,1}
      \foreach \dest in {1,...,1}
	\path (H2-\source) edge [ultra thick] (O-\dest);

 \foreach \source in {1,...,4}
     \foreach \dest in {2,...,4}
	\path (H2-\source) edge [ultra thick] (O-\dest);

 \foreach \source in {2,...,4}
      \foreach \dest in {1,...,1}
	\path (H2-\source) edge [ultra thick] (O-\dest);


\end{tikzpicture}
\caption{Fully-connected network with randomized input layer (red edges correspond to structured matrix). $k < n$. $\mathcal{D}$ is a random diagonal matrix with diagonal entries chosen independently from the binary set $\{-1,1\}$,
each value being chosen with probability $\frac{1}{2}$, and $\mathcal{P}$ is a structured matrix. \textit{The figure should be viewed in color.}}
\label{fig:autoenrsfc}
\vspace{-0.03in}
\end{figure}
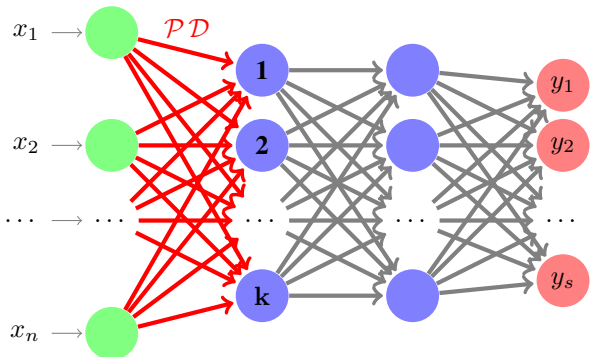

We performed experiments on \textit{MNIST} dataset downloaded from \url{http://yann.lecun.com/exdb/mnist/}. 
The data was preprocessed\footnote{Preprocessing is discussed in Section~\ref{sec:has_mech}.} according to the \textit{short} hashing scheme (the \textit{extended} hashing scheme 
gave results of no significant statistical difference) before being given to the input of the network. We first considered 
a simple model of the fully-connected feed-forward neural network with two hidden layers, where the first hidden layer 
had $k$ units that use sign non-linearity (we explored $k = \{16,32,64,128,256,512,1024\}$), and the second hidden layer 
had $100$ units that use ReLU non-linearity. The size of the second hidden layer was chosen as follows. We first investigated 
the dependence of the test error on this size in case when $n=k$ and the inputs instead of being randomly projected are 
multiplied by identity (it is equivalent to eliminating first hidden layer). We then chose as a size the threshold below which test 
performance was rapidly deteriorating. 

\begin{figure}[htp!]
\vspace{-0.15in}
\center
a)\includegraphics[width = 2.55in]{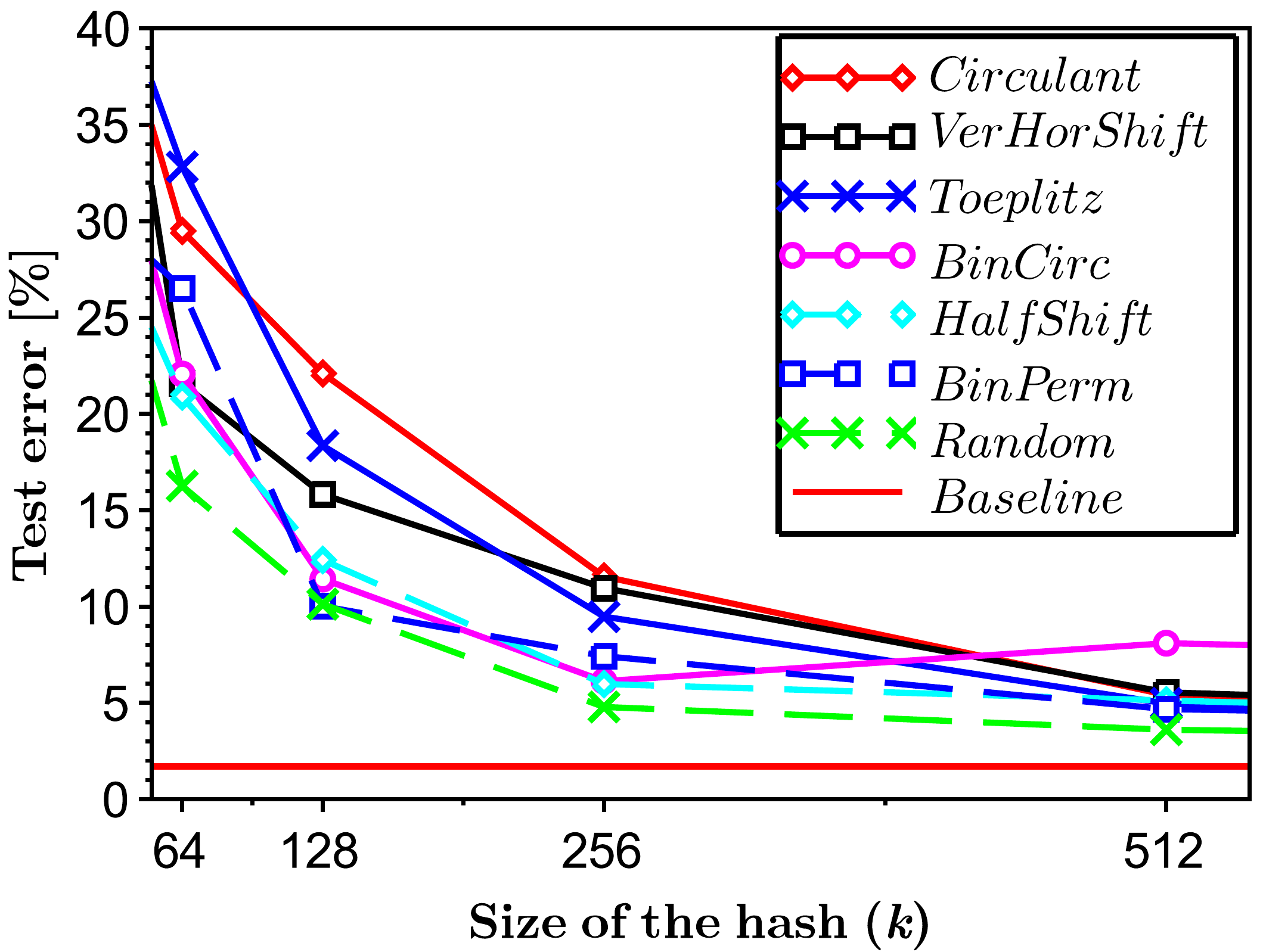}\\
b)\includegraphics[width = 2.55in]{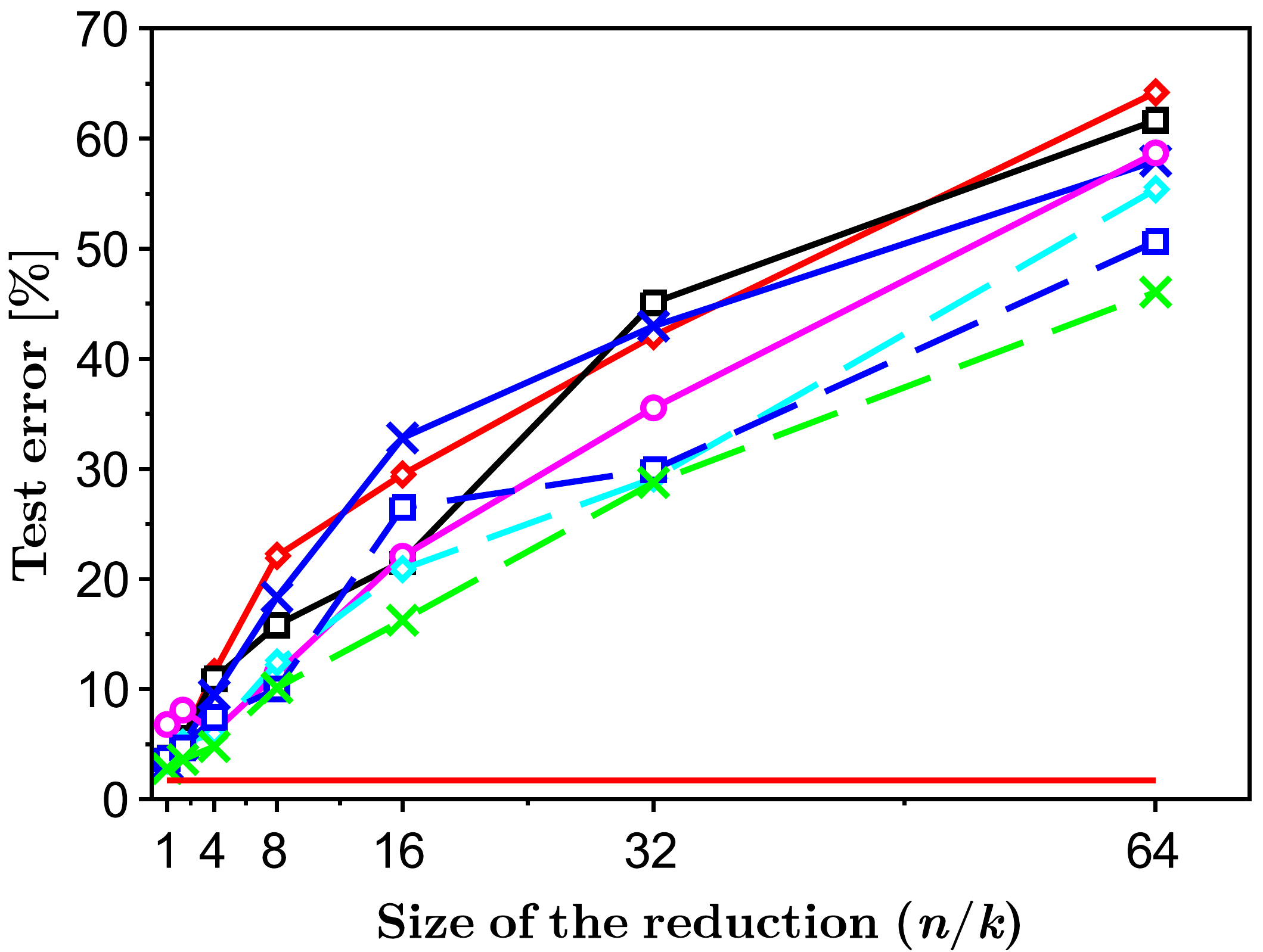}
\vspace{-0.1in}
\caption{Mean test error versus a) the size of the hash ($k$) (zoomed plot\footnotemark[3]), b) the size of the reduction ($n/k$) for the network. Baseline corresponds to $1.7\%$.}
\label{fig:plots}
\vspace{-0.08in}
\end{figure}

\begin{table*}[htp!]
\center
\setlength{\tabcolsep}{4.5pt}
\caption{Mean and std of the test error versus the size of the hash ($k$) / size of the reduction ($n/k$) for the network.}
\vspace{-0.1in}
\begin{tabular}{|c|c|c|c|c|c|c|c|c|c|c|c|}
\hline
\multicolumn{1}{|c|}{\multirow{2}{*}{$k\:\:$/$\:\:\frac{n}{k}$}} & Circulant & Random & BinPerm & BinCirc & HalfShift & Toeplitz & VerHorShift \\
& $[\%]$ & $[\%]$ & $[\%]$ &  $[\%]$ & $[\%]$ & $[\%]$ & $[\%]$\\
  \hline
\hline
$1024$ / $1$ & $3.53 \pm 0.16$ & $2.78 \pm 0.10$ & $3.69 \pm 0.21$ & $6.79 \pm 0.49$ & $3.54 \pm 0.16$ & $3.16 \pm 0.19$ & $3.74 \pm 0.16$\\
\hline
$512$ / $2$ & $5.42 \pm 0.83$ & $3.61 \pm 0.19$ & $4.68\pm 0.35$ & $8.10 \pm 1.85$ & $5.13 \pm 2.15$ & $4.97 \pm 0.53$ & $5.55 \pm 0.62$\\
\hline
$256$ / $4$ & $11.56 \pm 1.42$ & $4.79 \pm 0.13$ & $7.43 \pm 1.31$ & $6.13 \pm 1.42$ & $5.98 \pm 1.05$ & $9.48 \pm 1.88$ & $10.96 \pm 2.78$\\
\hline
$128$ / $8$ & $22.10 \pm 5.42$ & $10.13 \pm 0.24$ & $10.02 \pm 0.50$ & $11.43 \pm 0.92$ & $12.42 \pm 0.95$ & $18.35 \pm 2.36$ & $15.82 \pm 1.63$\\
\hline
$64$ / $16$ & $29.50 \pm 1.13$ & $16.26 \pm 1.02$ & $26.50 \pm 10.55$ & $22.07 \pm 1.35$ & $20.90 \pm 2.25$ & $32.82 \pm 4.83$ & $21.59 \pm 3.05$\\
\hline
$32$ / $32$ & $42.07 \pm 4.16$ & $28.77\pm 2.28$ & $29.94 \pm 3.48$ & $35.55 \pm 3.12$ & $29.15 \pm 0.97$ & $42.97 \pm 2.08$ & $45.10 \pm 4.46$\\
\hline
$16$ / $64$ & $64.20 \pm 6.76$ & $46.06\pm 1.03$ & $50.65 \pm 5.66$ & $58.70\pm 7.15$ & $55.40 \pm 6.90$ & $57.96 \pm 3.65$ & $61.66 \pm 4.08$\\
\hline
\end{tabular} 
\label{tab:one}
\end{table*}

\begin{table*}[htp!]
\vspace{-0.05in}
\center
\setlength{\tabcolsep}{1.2pt}
\caption{Memory complexity and number of required random values for structured matrices and \textit{Random} matrix.}
\vspace{-0.1in}
\begin{tabular}{|c||c|c|c|c|c|c|c|c|c|c|c|}
\hline
Matrix & Random & Circulant & BinPerm & HalfShift & VerHorShift & BinCirc & Toeplitz\\
  \hline
\hline
$\#$ of random values & $\mathcal{O}(nk)$ & $\mathcal{O}(n)$ & $\mathcal{O}(n)$ & $\mathcal{O}(n)$ & $\mathcal{O}(n)$ &  $\mathcal{O}(n)$ & $\mathcal{O}(n)$\\
\hline
Memory complexity& $\mathcal{O}(nk)$ & $\mathcal{O}(n)$ & $\mathcal{O}(nk)$ & $\mathcal{O}(n)$ & $\mathcal{O}(n)$ &  $\mathcal{O}(n)$ & $\mathcal{O}(n)$\\
\hline
\end{tabular} 
\label{tab:two}
\vspace{-0.1in}
\end{table*}

The first hidden layer contains random untrained weights, and we only train the parameters of the second layer and the output layer. 
The network we consider is shown in Figure~\ref{fig:autoenrsfc}. Each experiment was initialized from a random set of parameters sampled 
uniformly within the unit cube, and was repeated $1000$ times. All networks were trained for $30$ epochs using SGD~\cite{bottou-98x}. 
The experiments with constant learning rate are reported (we also explored learning rate decay, but obtained similar results), where the 
learning rate was chosen from the set $\{0.0005, 0.001, 0.002, 0.005, 0.01, 0.02, 0.05, 0.1 , 0.2, 0.5,$\\$1\}$ to minimize the test error. 
The weights of the first hidden layer correspond to the entries in the ``preprocessed'' structured matrix. We explored seven kinds of random matrices (first six are structured): \textit{Circulant}, \textit{Toeplitz}, \textit{HalfShift}, 
\textit{VerHorShift}, \textit{BinPerm}, \textit{BinCirc}, and \textit{Random} (entries are independent and drawn from Gaussian distribution $\mathcal{N}(0,1)$). 
All codes were implemented in Torch7.

\begin{figure}[h!]
\center
a)\includegraphics[width = 2.55in]{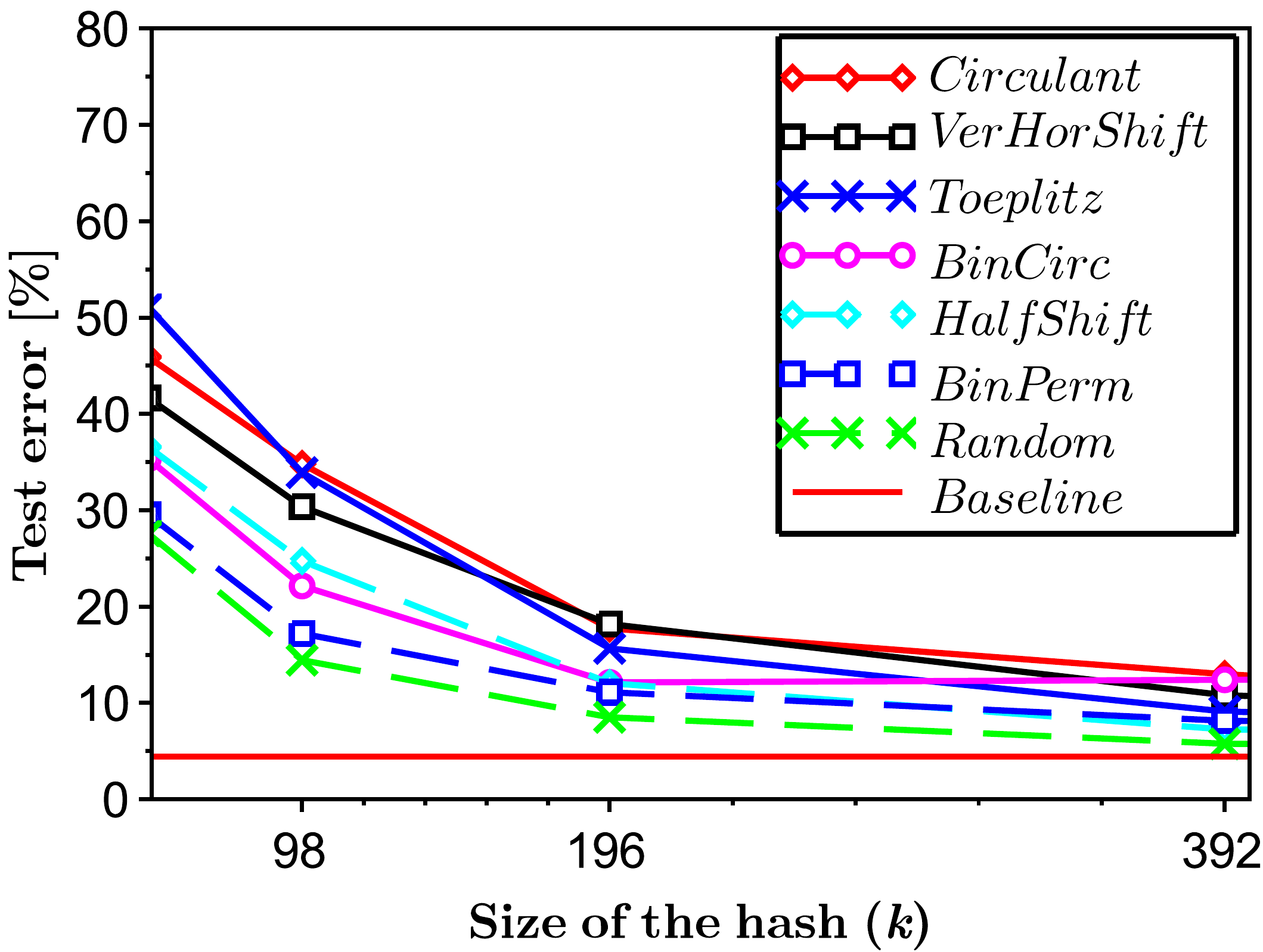}\\
b)\includegraphics[width = 2.55in]{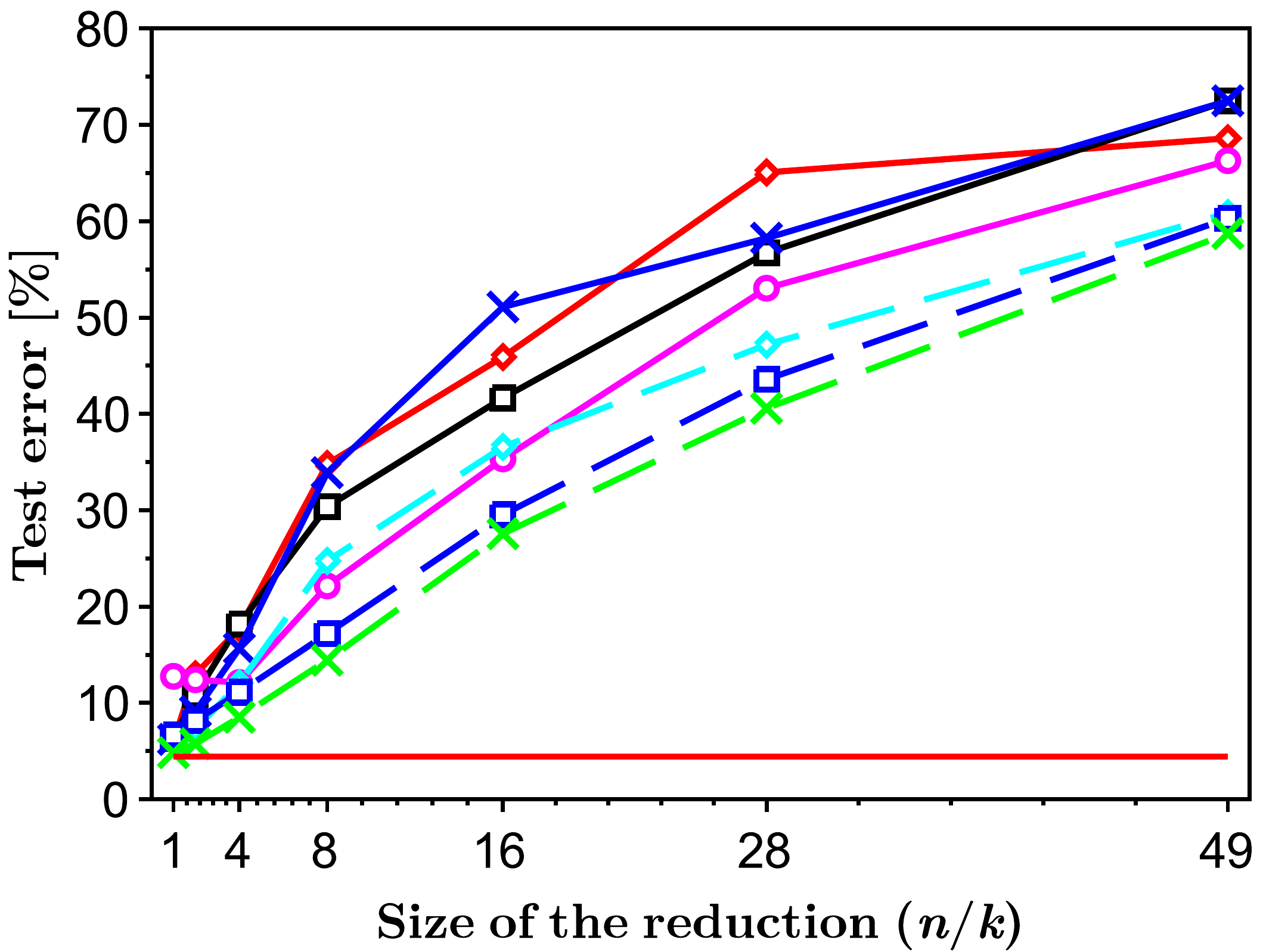}
\vspace{-0.05in}
\caption{Mean test error versus a) the size of the hash ($k$) (zoomed plot\footnotemark[2]), b) the size of the reduction ($n/k$) for $1$-NN. Baseline corresponds to $4.5\%$.}
\label{fig:plotskNN}
\vspace{-0.1in}
\end{figure}

Figure~\ref{fig:plots}a shows how the mean test error is affected by the size of the hash, and  Figure~\ref{fig:plots}b shows how the mean test error changes 
with the size of the reduction, where the size of the reduction is defined as the ratio $n/k$. In Table~\ref{tab:one} we report both the mean and the standard deviation (std) of the test error across our experiments. Training results are reported in the Supplementary material. 
\textit{Baseline} refers to the network with one hidden layer containing $100$ hidden units, where all parameters are trained.

\footnotetext[3]{Original plot is in the Supplement.}

Experimental results show how the performance is affected by using structured hashed projections to reduce data dimensionality. Figure~\ref{fig:plots}b and Table~\ref{tab:one} show close to linear dependence between the error and the size of the reduction. Simultaneously, this approach leads to computational savings and the reduction of memory storage. i.e. the reduction of the number of input weights for the hidden layer (in example for \textit{Circulant} matrix this reduction is of the order $\mathcal{O}(n/k)$\footnote[4]{The memory required for storing \textit{Circulant} matrix is negligible compared to the number of weights.}). Memory complexity, i.e. memory required to store the matrix, and the number of required random values for different structured matrices and \textit{Random} matrix are summarized in Table~\ref{tab:two}.

Experiments show that using fully random matrix gives the best performance as predicted in theory. \textit{BinPerm} matrix exhibits comparable performance to the \textit{Random} matrix, which might be explained by the fact that applying permutation itself adds an additional source of randomness. 
The next best performer is \textit{HalfShift}, whose generation is less random than the one of \textit{BinPerm} or \textit{Random}. Thus its performance, as expected, is worse than for these two other matrices. However, as opposed to \textit{BinPerm} and \textit{Random} matrices, \textit{HalfShift} matrix can be stored in linear space. The results also show that in general all structured matrices perform relatively well for medium-size reductions. Finally, all structured matrices except for \textit{BinPerm} lead to the biggest memory savings and require the smallest ``budget of randomness''. Moreover, they often lead to computational efficiency, e.g. \textit{Toeplitz} matrix-vector multiplications can be efficiently implemented via Fast Fourier Transform~\cite{felix}. But, as mentioned before, faster than naive matrix-vector product computations can be performed also for \textit{BinPerm}, \textit{HalfShift}, and \textit{VerHorShift}.

Finally, we also report how the performance of $1$-NN algorithm is affected by using structured hashed projections for the dimensionality reduction. We obtained similar plots as for the case of neural networks. They are captured in Figure~\ref{fig:plotskNN}. The table showing the mean and the standard deviation of the test error for experiments with $1$-NN is enclosed in the Supplementary material.

\section{Conclusions}
\label{sec:con}

This paper shows that structured hashed projections well preserve 
the angular distance between input data instances. Our theoretical results 
consider mapping the data to lower-dimensional space using various structured matrices, where the structured linear projections are
followed by the \textit{sign} nonlinearity. This non-linear operation was not 
considered for such a wide range of structured matrices in previous related theoretical works. The theoretical setting naturally applies 
to the multilayer network framework, where the basic components of the architecture perform 
matrix-vector multiplication followed by the nonlinear mapping. We empirically verify 
our theoretical findings and show how using structured hashed projections for dimensionality reduction
affects the performance of neural network and nearest neighbor classifier.

\bibliography{SH2016}
\bibliographystyle{icml2016}

\newpage

\clearpage

\toptitlebar 
{\Large \bf  \centering{Binary embeddings with structured hashed projections\\(Supplementary Material)} \par}
\bottomtitlebar

In this section we prove Theorem \ref{ext_technical_theorem} and Theorem \ref{short_theorem}.
We will use notation from Lemma \ref{mean_lemma}.

\section{Proof of Theorem \ref{ext_technical_theorem}}

We start with the following technical lemma:

\begin{lemma}
\label{first_lemma}
Let $\{Z_{1},...,Z_{k}\}$ be the set of $k$ independent random variables defined on $\Omega$ such that each $Z_{i}$ has the same distribution and $Z_{i} \in \{0,1\}$. Let $\{\mathcal{F}_{1},...,\mathcal{F}_{k}\}$ be the set of events, where each $\mathcal{F}_{i}$ is in the $\sigma$-field defined by $Z_{i}$ (in particular $\mathcal{F}_{i}$ does not depend on the $\sigma$-field $\sigma(Z_{1},...,Z_{i-1},Z_{i+1},...Z_{k})$). Assume that there exists $\mu < \frac{1}{2}$ such that: $\mathbb{P}(\mathcal{F}_{i}) \leq \mu$ for $i=1,...,k$.
Let $\{U_{1},...,U_{k}\}$ be the set of $k$ random variables such that $U_{i} \in \{0,1\}$ and
$U_{i} | \mathcal{F}_{i} = Z_{i} |\mathcal{F}_{i}$ for $i=1,...,k$, where $X|\mathcal{F}$ stands
for the random variable $X$ truncated to the event $\mathcal{F}$. Assume furthermore that $E(U_{i})=E(Z_{i})$ for $i=1,...,k$.
Denote $U = \frac{U_{1}+...+U_{k}}{k}$. Then the following is true.
\begin{equation}
\mathbb{P}(|U-EU| \!>\! \epsilon) \!\leq\! \frac{1}{\pi} \!\sum_{r=\frac{\epsilon k}{2}}^{k}\!\!\frac{1}{\sqrt{r}}(\frac{ke}{r})^{r}\mu^{r}(1-\mu)^{k-r} + 2e^{-\frac{\epsilon^{2}k}{2}}.
\end{equation}
\end{lemma}

\begin{proof}
Let us consider the event $\mathcal{F}_{bad}$ = $\mathcal{F}_{1} \cup ... \cup \mathcal{F}_{k}$.
Note that $\mathcal{F}_{bad}$ may be represented by the union of the so-called $r$-blocks, i.e.
\begin{equation}\mathcal{F}_{bad} = \bigcup_{Q \subseteq \{1,...,k\}} (\bigcap_{q \in Q} \mathcal{F}_{q} \bigcap_{q \in \{1,...,k\} \setminus Q} \mathcal{F}^{c}_{q}), \end{equation}

where $\mathcal{F}^{c}$ stands for the complement of event $\mathcal{F}$.
Let us fix now some $Q \subseteq \{1,...,k\}$. Denote \begin{equation}\mathcal{F}_{Q} = \bigcap_{q \in Q} \mathcal{F}_{q} \bigcap_{q \in \{1,...,k\} \setminus Q} \mathcal{F}^{c}_{q}. \end{equation}
note that $\mathbb{P}(\mathcal{F}_{Q}) \leq \mu^{r}(1-\mu)^{k-r}$. It follows directly from the Bernoulli scheme. 

Denote $Z = \frac{Z_{1}+...+Z_{k}}{k}$. 
From what we have just said and from the definition of $\{\mathcal{F}_{1},...,\mathcal{F}_{k}\}$ we conclude that for any given $c$ the following holds:
\begin{equation}
\label{xy_diff}
\mathbb{P}(|U-Z| > c) \leq \sum_{r=ck}^{k}{k \choose r}\mu^{r}(1-\mu)^{k-r}.
\end{equation}

Note also that from the assumptions of the lemma we trivially get: $E(U)=E(Z)$.

Let us consider now the expression $\mathbb{P}(|U-E(U)|) > \epsilon$. 

We get: $\mathbb{P}(|U-E(U)|> \epsilon) = 
\mathbb{P}(|U-E(Z)| > \epsilon) = \mathbb{P}(|U-Z + Z-E(Z)| > \epsilon) \leq \mathbb{P}(|U-Z|+|Z-E(Z)|>\epsilon) \leq \mathbb{P}(|U-Z| > \frac{\epsilon}{2}) + \mathbb{P}(|Z-E(Z)| > \frac{\epsilon}{2})$.

From \ref{xy_diff} we get:

\begin{equation}
\mathbb{P}(|U-Z| > \frac{\epsilon}{2}) \leq \sum_{r=\frac{\epsilon}{2}}^{k}{k \choose r} \mu^{r}(1-\mu)^{k-r}.
\end{equation}

Let us consider now the expression: \begin{equation}\xi = \sum_{r=\frac{\epsilon k}{2}}^{k}{k \choose r} \mu^{r}(1-\mu)^{k-r}.\end{equation}
We have:
\begin{eqnarray}
\xi &\leq& \sum_{r=\frac{\epsilon k}{2}}^{k} \frac{(k-r+1)...(k)}{r!} \mu^{r}(1-\mu)^{k-r} \nonumber\\
&\leq& \sum_{r=\frac{\epsilon k}{2}}^{k} \frac{k^{r}}{r!} \mu^{r}(1-\mu)^{k-r}
\end{eqnarray}
From the Stirling's formula we get: $r! = \frac{2\pi r^{r+\frac{1}{2}}}{e^{r}}(1+o_{r}(1))$.
Thus we obtain:
\begin{eqnarray}
\label{xi_ineq}
\xi &\leq& (1+o_{r}(1))\sum_{r=\frac{\epsilon k}{2}}^{k}\frac{k^{r}e^{r}}{2\pi r^{r+\frac{1}{2}}}\mu^{r}(1-\mu)^{k-r} \nonumber\\
&\leq& \frac{1}{\pi} \sum_{r=\frac{\epsilon k}{2}}^{k}\frac{1}{\sqrt{r}}(\frac{ke}{r})^{r}\mu^{r}(1-\mu)^{k-r}
\end{eqnarray}
for $r$ large enough.

Now we will use the following version of standard Azuma's inequality:
\begin{lemma}
\label{azuma_general}
Let $W_{1},...,W_{k}$ be $k$ independent random variables such that $E(W_{1})=...E(W_{k})=0$.
Assume that $-\alpha_{i} \leq W_{i+1} - W_{i} \leq \beta_{i}$ for $i=2,...,k-1$.
Then the following is true:
$$
\mathbb{P}(|\sum_{i=1}^{k} W_{i}|>a) \leq 2e^{-\frac{2a^{2}}{\sum_{i=1}^{k}(\alpha_{i}+\beta_{i})^{2}}}
$$
\end{lemma}

Now, using Lemma \ref{azuma_general} for $W_{i} = X_{i} - E(X_{i})$ and $\alpha_{i} = E(X_{i}), \beta_{i}=1-E(X_{i})$ we obtain:
\begin{equation}
\label{azuma_simple}
\mathbb{P}(|X-EX| > \frac{a}{2}) \leq 2e^{-\frac{a^{2}k}{2}}.
\end{equation}

Combining \ref{xi_ineq} and \ref{azuma_simple}, we obtain the statement of the lemma.

\end{proof}

Our next lemma explains the role the Hadamard matrix plays in the entire extended $\Psi$-regular hashing mechanism.

\begin{lemma}
\label{hadamard_lemma}
Let $n$ denote data dimensionality and let $f(n)$ be an arbitrary positive function.
Let $D$ be the set of all $L_{2}$-normalized data points, where no two data points are identical.
Assume that $|D|=N$.
Consider the ${N \choose 2}$ hyperplanes $H_{p,r}$ spanned by pairs of different vectors $\{p,r\}$ from $D$. Then after applying linear transformation $\mathcal{H}\mathcal{R}$ each hyperplane $H_{p,r}$ is transformed into another hyperplane $H^{\mathcal{H}\mathcal{R}}_{p,r}$. Furthermore, the probability $\mathcal{P}_{\mathcal{H}\mathcal{R}} $that for every $H^{\mathcal{H}\mathcal{R}}_{p,r}$ there exist two orthonormal vectors $x=(x_{1},...,x_{n}),y=(y_{1},...,y_{n})$ in $H^{\mathcal{H}\mathcal{R}}_{p,r}$ such that: $|x_{i}|,|y_{i}| \leq \frac{f(n)}{\sqrt{n}}$
satisfies: $$\mathcal{P}_{\mathcal{H}\mathcal{R}}  \geq 1-4{N \choose 2}e^{-\frac{f^{2}(n)}{2}}.$$
\end{lemma}

\begin{proof}
We have already noted in the proof of Lemma \ref{mean_lemma} that $\mathcal{H}\mathcal{R}$
is an orthogonal matrix. Thus, as an isometry, it clearly transforms each $2$-dimensional hyperplane into another $2$-dimensional hyperplane.
For every pair $\{p,r\}$, let us consider an arbitrary fixed orthonormal pair $\{u,v\}$ spanning $H_{p,r}$.
Denote $u=(u_{1},...,u_{n})$. Let us denote by $u^{\mathcal{H}\mathcal{R}}$ vector obtained from 
$u$ after applying transformation $\mathcal{H}\mathcal{R}$.
Note that the $j^{th}$ coordinate of $u^{\mathcal{H}\mathcal{R}}$ is of the form:
\begin{equation}
u^{\mathcal{H}\mathcal{R}}_{j} = u_{1}T_{1}+...+u_{n}T_{n},
\end{equation}
where $T_{1},...,T_{n}$ are independent random variables satisfying:

\begin{equation}
T_{i} =
\left\{
	\begin{array}{ll}
		\frac{1}{\sqrt{n}}  & \mbox{w.p }  \frac{1}{2}, \\
		-\frac{1}{\sqrt{n}} & \mbox{otherwise.} 
	\end{array}
\right.
\end{equation}

The latter comes straightforwardly from the form of the $L_{2}$-normalized Hadamard matrix
(i.e a Hadamard matrix, where each row and column is $L_{2}$-normalized).

But then, from Lemma \ref{azuma_general}, and the fact that $\|u\|_{2}=1$, we get for any $a>0$:

\begin{equation}
\mathbb{P}(|u_{1}T_{1}+...+u_{n}T_{n}| \geq a) \leq 2e^{-\frac{2a^{2}}{\sum_{i=1}^{n}(2u_{i})^{2}}} \leq 2e^{-\frac{a^{2}}{2}}.
\end{equation}

Similar analysis is correct for $v^{\mathcal{H}\mathcal{R}}$.
Note that $v^{\mathcal{H}\mathcal{R}}$ is orthogonal to $u^{\mathcal{H}\mathcal{R}}$
since $v$ and $u$ are orthogonal. Furthermore, both $v^{\mathcal{H}\mathcal{R}}$ and 
$u^{\mathcal{H}\mathcal{R}}$ are $L_{2}$-normalized. Thus $\{u^{\mathcal{H}\mathcal{R}},v^{\mathcal{H}\mathcal{R}}\}$ is an orthonormal pair.

To complete the proof, it suffices to take $a=f(n)$ and apply the union bound over all
vectors $u^{\mathcal{H}\mathcal{R}}$, $v^{\mathcal{H}\mathcal{R}}$ for all  ${N \choose 2}$
hyperplanes.
\end{proof}

From the lemma above we see that applying Hadamard matrix enables us to assume with high probability
that for every hyperplane $H_{p,r}$ there exists an orthonormal basis consisting of vectors with elements of absolute values at most $\frac{f(n)}{\sqrt{n}}$. We call this event $\mathcal{E}_{f}$. Note that whether $\mathcal{E}_{f}$ holds or not is determined only by $\mathcal{H}$, $\mathcal{R}$ and the initial dataset $D$.

Let us proceed with the proof of Theorem \ref{ext_technical_theorem}.
Let us assume that event $\mathcal{E}_{f}$ holds. Without loss of generality we may assume that 
we have the short $\Psi$-regular hashing mechanism with an extra property that every $H_{p,r}$ has an orthonormal basis consisting of vectors with elements of absolute value at most $\frac{f(n)}{\sqrt{n}}$.
Fix two vectors $p,r$ from the dataset $D$. Denote by $\{x,y\}$ the orthonormal basis of $H_{p,r}$ with the above property. Let us fix the $i$th row of $\mathcal{P}$ and denote it as $(p_{i,1},...,p_{i,n})$.
After being multiplied by the diagonal matrix $\mathcal{D}$ we obtain another vector:
\begin{equation}
w=(\mathcal{P}_{i,1}d_{1},...,\mathcal{P}_{i,n}d_{n}),
\end{equation}

where:

\begin{equation}
\mathcal{D}_{i,j} =
 \begin{pmatrix}
  d_{1} & 0 & \cdots & 0 \\
  0 & d_{2} & \cdots & 0 \\
  \vdots  & \vdots  & \ddots & \vdots  \\
  0 & 0 & \cdots & d_{n}
 \end{pmatrix}.
\end{equation}

We have already noted that in the proof of Lemma \ref{mean_lemma} that it is the projection of $w$ into $H_{p,r}$ that determines whether the value of the associated random variable $X_{i}$ is $0$ or $1$.
To be more specific, we showed that $X_{i}=1$ iff the projection is in the region $\mathcal{U}_{p,r}$. 
Let us write down the coordinates of the projection of $w$ into $H_{p,r}$ in the $\{x,y\}$-coordinate system.
The coordinates are the dot-products of $w$ with $x$ and $y$ respectively thus in the $\{x,y\}$-coordinate system we can write $w$ as:
\begin{equation}
\label{coord_eq}
w_{\{x,y\}}=(\mathcal{P}_{i,1}d_{1}x_{1},...,\mathcal{P}_{i,n}d_{n}x_{n},\mathcal{P}_{i,1}d_{1}y_{1},...,\mathcal{P}_{i,n}d_{n}y_{n}).
\end{equation} 

Note that both coordinates are Gaussian random variables and they are independent since they were constructed by projecting a Gaussian vector into two orthogonal vectors.
Now note that from our assumption about the structure of $\mathcal{P}$ we can conclude that
both coordinates may be represented as sums of weighted Gaussian random variables $g_{i}$ for $i=1,...,t$, i.e.:
\begin{equation}
w_{\{x,y\}}=(g_{1}s_{i,1}+...+g_{t}s_{i,t},g_{1}v_{i,1}+...+g_{t}v_{i,t}),
\end{equation}

where each $s_{i,j}, v_{i,j}$ is of the form $d_{z}x_{z}$ or $d_{z}y_{z}$ for some $z$ that 
depends only on $i,j$. 
Note also that 
\begin{equation}
s_{i,1}^{2}+...+s_{i,t}^{2} =v_{i,1}^{2}+...+v_{i,t}^{2}.
\end{equation}
The latter inequality comes from the fact that, by \ref{coord_eq}, both coordinates of 
$w_{\{x,y\}}$ have the same distribution.

Let us denote $s_{i}=(s_{i,1},...,s_{i,t})$, $v_{i}=(v_{i,1},...,v_{i,t})$ for $i=1,...,k$.
We need the following lemma stating that with high probability vectors $s_{1},...,s_{k},v_{1},...,v_{k}$
are close to be pairwise orthogonal.

\begin{lemma}
\label{small_dot_product_lemma}
Let us assume that $\mathcal{E}_{f}$ holds. Let $f(n)$ be an arbitrary positive function. Then for every $a>0$ with probability at least 
$\mathbb{P}_{succ} \geq 1 - 4{k \choose 2} e^{-\frac{2a^{2}n}{f^{4}(n)}}$, taken under coin tosses used to construct $\mathcal{D}$, the following is true for every $1 \leq i_{1} \neq i_{2} \leq k$:
\label{pseudo_ortho_lemma}
$$|\sum_{u=1}^{n} s_{i_{1},u}v_{i_{1},u}| \leq a\chi(\mathcal{P}) + \Psi \frac{f^{2}(n)}{n},$$
$$|\sum_{u=1}^{n} s_{i_{1},u}s_{i_{2},u}| \leq a\chi(\mathcal{P}) + \Psi \frac{f^{2}(n)}{n},$$
$$|\sum_{u=1}^{n} v_{i_{1},u}v_{i_{2},u}| \leq a\chi(\mathcal{P}) + \Psi \frac{f^{2}(n)}{n},$$
$$|\sum_{u=1}^{n} s_{i_{1},u}v_{i_{2},u}| \leq a\chi(\mathcal{P}) + \Psi \frac{f^{2}(n)}{n}.$$
\end{lemma}

\begin{proof}
Note that the we get the first inequality for free from the fact that $x$ is orthogonal to $y$
(in other words, $\sum_{u=1}^{n} s_{i_{1},u}v_{i_{1},u}$ can be represented as $C\sum_{u=1}^{n} x_{i}y_{i}$ and the latter expression is clearly $0$). 
Let us consider now one of the three remaining expressions. Note that they can be rewritten as:
\begin{equation}E = \sum_{i=1}^{n} d_{\rho(i)}d_{\lambda(i)} x_{\zeta(i)}x_{\gamma(i)}\end{equation} 
or \begin{equation}E = \sum_{i=1}^{n} d_{\rho(i)}d_{\lambda(i)} y_{\zeta(i)}y_{\gamma(i)}\end{equation}
or \begin{equation}E = \sum_{i=1}^{n} d_{\rho(i)}d_{\lambda(i)} x_{\zeta(i)}y_{\gamma(i)}\end{equation} for some
$\rho, \lambda, \zeta, \gamma$. 
Note also that from the $\Psi$-regularity condition we immediately obtain that $\rho(i)=\lambda(i)$
for at most $\Psi$ elements of each sum. Get rid of these elements from each sum and consider the remaining ones. From the definition of the $\mathcal{P}$-chromatic number, those remaining ones can be partitioned into at most $\chi(\mathcal{P})$ parts, each consisting of elements that are independent random variables (since in the corresponding graph there are no edges between them).
Thus, for the sum corresponding to each part one can apply Lemma \ref{azuma_general}.
Thus one can conclude that the sum differs from its expectation (which clearly is $0$ since $E(d_{i}d_{j})=0$ for $i \neq j$) by a with probability at most: 
\begin{equation}
\mathbb{P}_{a} \leq 2e^{-\frac{2a^{2}}{\sum_{i=1}^{n} x_{\zeta(i)}x_{\gamma(i)}}},
\end{equation}
or
\begin{equation}
\mathbb{P}_{a} \leq 2e^{-\frac{2a^{2}}{\sum_{i=1}^{n} y_{\zeta(i)}y_{\gamma(i)}}},
\end{equation}
or
\begin{equation}
\mathbb{P}_{a} \leq 2e^{-\frac{2a^{2}}{\sum_{i=1}^{n} x_{\zeta(i)}y_{\gamma(i)}}}.
\end{equation}

Now it is time to use the fact that event $\mathcal{E}_{f}$ holds.
Then we know that: $|x_{i}|,|y_{i}| \leq \frac{f(n)}{\sqrt{n}}$ for $i=1,...,n$.
Substituting this upper bound for $|x_{i}|,|y_{i}|$ in the derived expressions on the probabilities coming from Lemma \ref{azuma_general}, and then taking the union bound, we complete the proof.
\end{proof}

We can finish the proof of Theorem \ref{ext_technical_theorem}.
From Lemma \ref{small_dot_product_lemma} we see that $s_{1},...,s_{k},v_{1},...,v_{k}$ are
close to pairwise orthogonal with high probability. Let us fix some positive function $f(n)>0$ and some
$a>0$. Denote 

\begin{equation}
\Delta = a\chi(\mathcal{P}) + \Psi \frac{f^{2}(n)}{n}.
\end{equation}

Note that, by Lemma \ref{small_dot_product_lemma} we see that applying Gram-Schmidt process
we can obtain a system of pairwise orthogonal vectors $\tilde{s}_{1},...,\tilde{s}_{k},\tilde{v}_{1},...,\tilde{v}_{k}$ such that 
\begin{equation} \label{ineq1}\|\tilde{v}_{i}-v_{i}\|_{2} \leq \sigma(k) \Delta. \end{equation}
and 
\begin{equation}\label{ineq2}\|\tilde{s}_{i}-s_{i}\|_{2} \leq \sigma(k) \Delta, \end{equation}
where $\sigma(k)>0$ is some function of $k$ (it does not depend on $n$ and $t$).
Note that for $n,t$ large enough we have: $\sigma(k) \Delta \leq \sqrt{a} \chi(\mathcal{P}) + \Psi \frac{f^{2}(n)}{\sqrt{n}}$. 

Let us consider again $w_{x,y}$.  Replacing $s_{i}$ by $\tilde{s}_{i}$ and $v_{i}$ by $\tilde{v}_{i}$
in the formula on $w_{x,y}$, we obtain another Gaussian vector: $\tilde{w}_{x,y}$ for each row $i$ of the matrix $\mathcal{P}$. Note however that vectors $\tilde{w}_{x,y}$ have one crucial advantage over vectors $w_{x,y}$, namely they are independent. That comes from the fact that $\tilde{s}_{1},...,\tilde{s}_{k}$,$\tilde{v}_{1},...,\tilde{v}_{k}$ are pairwise orthogonal.
Note also that from \ref{ineq1} and \ref{ineq2} we obtain that the angular distance between 
 $w_{x,y}$ and $\tilde{w}_{x,y}$ is at most $\sigma(k)\Delta$.

Let $Z_{i}$ for $i=1,...k$ be an indicator random variable that is zero if $\tilde{w}_{x,y}$ is inside the region $\mathcal{U}_{p,r}$ and zero otherwise. 
Let $U_{i}$ for $i=1,...k$ be an indicator random variable that is zero if $w_{x,y}$ is inside the region $\mathcal{U}_{p,r}$ and zero otherwise. 
Note that $\tilde{\theta}^{n}_{p,r} = \frac{U_{1}+...+U_{k}}{k}$.
Furthermore, random variables $Z_{1},...,Z_{k},U_{1},...,U_{k}$ satisfy the assumptions of 
Lemma \ref{first_lemma} with $\mu \leq \frac{8\tau}{\theta_{p,r}}$, where $\tau = \sigma(k)\Delta$.
Indeed,  random variables $Z_{i}$ are independent since vectors $\tilde{w}_{x,y}$ are independent.
From what we have said so far we know that each of them takes value one with probability exactly $\frac{\theta_{p,r}}{\pi}$.
Furthermore $Z_{i} \neq U_{i}$ only if $w_{x,y}$ is inside $\mathcal{U}_{p,r}$ and $\tilde{w}_{x,y}$
is outside $\mathcal{U}_{p,r}$ or vice versa. The latter event implies (thus it is included in the event) 
that $w_{x,y}$ is near the border of the region $\mathcal{U}_{p,r}$, namely within an angular distance $\frac{\epsilon}{\theta_{p,r}}$ 
from one of the four semi-lines defining $\mathcal{U}_{p,r}$. Thus in particular, an event $Z_{i} \neq U_{i}$ is contained in 
the event of probability at most $2 \cdot 4 \cdot \frac{\epsilon}{\theta_{p,r}}$ that depends only on one $w_{x,y}$.

But then we can apply Lemma \ref{first_lemma}. All we need is to assume that the premises of Lemma \ref{small_dot_product_lemma} are satisfied. But this is the case with probability specified in Lemma \ref{hadamard_lemma} and this probability is taken under random coin tosses used to product $\mathcal{H}$ and $\mathcal{R}$, thus independently from the random coin tosses used to produce $\mathcal{D}$. Putting it all together we obtain the statement of Theorem \ref{ext_technical_theorem}.

\section{Proof of Theorem \ref{short_theorem}}

We will borrow some notation from the proof of Theorem \ref{ext_technical_theorem}.
Note however that in this setting no preprocessing with the use of matrices $\mathcal{H}$
and $\mathcal{R}$ is applied.

\begin{lemma}
\label{variance_lemma}
Define $U_{1},...,U_{k}$ as in the proof of Theorem \ref{ext_technical_theorem}.
Assume that the following is true:
$$|\sum_{u=1}^{n} s_{i_{1},u}v_{i_{1},u}| \leq \Delta,$$
$$|\sum_{u=1}^{n} s_{i_{1},u}s_{i_{2},u}| \leq \Delta,$$
$$|\sum_{u=1}^{n} v_{i_{1},u}v_{i_{2},u}| \leq \Delta,$$
$$|\sum_{u=1}^{n} s_{i_{1},u}v_{i_{2},u}| \leq \Delta.$$ 
for some $0<\Delta<1$.
The the following is true for every fixed $1 \leq i < j \leq k$:
$$|\mathbb{P}(U_{i}U_{j}=1) - \mathbb{P}(U_{i}=1)\mathbb{P}(U_{j}=1)| = O(\Delta).$$
\end{lemma}

The lemma follows from the exactly the same analysis that was done in the last section of the proof of Theorem \ref{ext_technical_theorem} thus we leave it to the reader as an exercise.

Note that we have: 

\begin{eqnarray}
\!\!\!\!\!\!\!\!\!Var(\tilde{\theta}^{n}_{p,r}) \!\!\!\!\!&=&\!\!\!\!\! Var(\frac{U_{1}+...+U_{k}}{k}) \nonumber\\
\!\!\!\!\!&=&\!\!\!\!\! \frac{1}{k^{2}}(\sum_{i=1}^{k} Var(U_{i}) + \sum_{i \neq j} Cov(U_{i},U_{j})).
\end{eqnarray} 

Since $U_{i}$ is an indicator random variable that takes value one with probability $\frac{\theta_{p,r}}{\pi}$,
we get:
\begin{equation}
Var(U_{i}) = E(U_{i}^{2}) - E(U_{i})^{2} = \frac{\theta_{p,r}}{\pi}(1-\frac{\theta_{p,r}}{\pi}).
\end{equation}

Thus we have:

\begin{equation}
Var(\tilde{\theta}^{n}_{p,r}) = \frac{1}{k}\frac{\theta_{p,r}(\pi-\theta_{p,r})}{\pi^{2}}+\frac{1}{k^{2}}\sum_{i \neq j} Cov(U_{i},U_{j}).
\end{equation}

Note however that $Cov(U_{i},U_{j})$ is exactly: $\mathbb{P}(U_{i}U_{j}=1) - \mathbb{P}(U_{i}=1)\mathbb{P}(U_{j}=1)$.

Therefore, using Lemma \ref{variance_lemma}, we obtain:

\begin{equation}
Var(\tilde{\theta}^{n}_{p,r})  = \frac{1}{k}\frac{\theta_{p,r}(\pi-\theta_{p,r})}{\pi^{2}} + O(\Delta).
\end{equation}

It suffices to estimate parameter $\Delta$.
We proceed as in the previous proof. 
We only need to be a little bit more cautious since the condition: $|x_{i}|,|y_{i}| \leq \frac{f(n)}{\sqrt{n}}$ cannot be assumed right now.
We select two rows: $i_{1},i_{2}$ of $\mathcal{P}$.
Note that again we see that applying Gram-Schmidt process,
we can obtain a system of pairwise orthogonal vectors $\tilde{s}_{i_{1}},\tilde{s}_{i_{i}},\tilde{v}_{i_{i}},\tilde{v}_{i_{2}}$ such that 
\begin{equation} \label{ineq1}\|\tilde{v}_{i_{1}}-v_{i_{2}}\|_{2} = O(\Delta), \end{equation}
and 
\begin{equation}\label{ineq2}\|\tilde{s}_{i_{1}}-s_{i_{2}}\|_{2} = O(\Delta). \end{equation}

The fact that right now the above upper bounds are not multiplied by $k$, as it was the case in the previous proof,
plays key role in obtaining nontrivial concentration results even when no Hadamard mechanism is applied.

\begin{table*}[htp!]
\center
\setlength{\tabcolsep}{4.5pt}
\caption{Mean and std of the train error versus the size of the hash ($k$) / size of the reduction ($n/k$) for the network.}
\begin{tabular}{|c|c|c|c|c|c|c|c|c|c|c|c|}
\hline
\multicolumn{1}{|c|}{\multirow{2}{*}{$k\:\:$/$\:\:\frac{n}{k}$}} & Circulant & Random & BinPerm & BinCirc & HalfShift & Toeplitz & VerHorShift \\
& $[\%]$ & $[\%]$ & $[\%]$ &  $[\%]$ & $[\%]$ & $[\%]$ & $[\%]$\\
  \hline
\hline
$1024$ / $1$ & $0.00 \pm 0.00$ & $0.00\pm 0.00$ & $0.00 \pm 0.00$ & $0.30 \pm 0.44$ & $0.00 \pm 0.00$ & $0.00 \pm 0.00$ & $0.00 \pm 0.00$\\
\hline
$512$ / $2$ & $0.04 \pm 0.06$ & $0.00\pm 0.00$ & $0.00 \pm 0.00$ & $2.66 \pm 2.98$ & $1.44 \pm 2.89$ & $0.00 \pm 0.00$ & $0.00 \pm 0.01$\\
\hline
$256$ / $4$ & $6.46 \pm 2.27$ & $0.00\pm 0.00$ & $0.79 \pm 1.57$ & $0.60 \pm 1.19$ & $0.49 \pm 0.93$ & $2.09 \pm 1.69$ & $3.98 \pm 3.96$\\
\hline
$128$ / $8$ & $16.89 \pm 6.57$ & $4.69\pm 0.43$ & $4.44 \pm 0.50$ & $5.62 \pm 1.03$ & $7.34 \pm 1.27$ & $11.82 \pm 2.17$ & $10.51 \pm 1.27$\\
\hline
$64$ / $16$ & $26.47 \pm 0.98$ & $13.35\pm 0.61$ & $23.98 \pm 11.54$ & $18.68 \pm 0.78$ & $17.64 \pm 2.01$ & $29.97 \pm 5.29$ & $18.68 \pm 3.26$\\
\hline
$32$ / $32$ & $40.79 \pm 3.82$ & $27.51\pm 2.04$ & $28.28 \pm 3.23$ & $33.91 \pm 3.23$ & $27.90 \pm 1.05$ & $41.49 \pm 2.14$ & $43.51 \pm 3.78$\\
\hline
$16$ / $64$ & $63.96 \pm 5.62$ & $46.31\pm 0.73$ & $50.03 \pm 6.18$ & $58.71 \pm 6.96$ & $54.88 \pm 6.47$ & $57.72 \pm 3.42$ & $60.91 \pm 4.53$\\
\hline
\end{tabular} 
\label{tab:one2}
\end{table*}

We consider the related sums:
$E_{1} = \sum_{i=1}^{n} d_{\rho(i)}d_{\lambda(i)} x_{\zeta(i)}x_{\gamma(i)},
  E_{2} = \sum_{i=1}^{n} d_{\rho(i)}d_{\lambda(i)} y_{\zeta(i)}y_{\gamma(i)},
  E_{3} = \sum_{i=1}^{n} d_{\rho(i)}d_{\lambda(i)} x_{\zeta(i)}y_{\gamma(i)}$ 
as before. We can again partition each sum into at most $\chi(\mathcal{P})$ sub-chunks, where
this time $\chi(\mathcal{P}) \leq 3$ (since $\mathcal{P}$ is Toeplitz Gaussian).
The problem is that applying Lemma \ref{azuma_general}, we get bounds that depend on the expressions of the form \begin{equation} \alpha_{x,i} = \sum_{j=1}^{n}x_{j}^{2}x_{j+i}^{2},\end{equation} and \begin{equation} \alpha_{y,i} = \sum_{j=1}^{n}y_{j}^{2}y_{j+i}^{2}, \end{equation} where indices are added modulo $n$ and this time we cannot assume that all $|x_{i}|,|y_{i}|$ are small.
Fortunately we have:
\begin{equation}
\sum_{i=1}^{n} \alpha_{x,i} = 1,
\end{equation}
and 
\begin{equation}
\sum_{i=1}^{n} \alpha_{y,i} = 1
\end{equation}

Let us fix some positive function $f(k)$. We can conclude that the number of variables $\alpha_{x,i}$
such that $\alpha_{x,i} \geq \frac{f(k)}{{k \choose 2}}$ is at most $\frac{{k \choose 2}}{f(k)}$.
Note that each such $\alpha_{x,i}$ and each such $\alpha_{y,i}$ corresponds to a pair $\{i_{1},_{2}\}$ of rows of the matrix $\mathcal{P}$ and consequently to the unique element $Cov(U_{i_{1}},U_{i_{2}})$ of the entire covariance sum (scaled by $\frac{1}{k^{2}}$).
Since trivially we have $|Cov(U_{i_{1}},U_{i_{2}})|=O(1)$, we conclude that the contribution of these elements to the entire covariance sum is of order $\frac{1}{f(k)}$.
Let us now consider these $\alpha_{x,i}$ and $\alpha_{y,i}$ that are at most $\frac{f(k)}{{k \choose 2}}$.
These sums are small (if we take $f(k)=o(k^{2})$) and thus it makes sense to apply Lemma \ref{azuma_general} to them. That gives us upper bound $a=\Omega(\Delta)$ with probability: 
\begin{equation}\mathbb{P}^{*} \geq 1-e^{-\Omega(a^{2}\frac{k^{2}}{f(k)})}.\end{equation}
Taking  $f(k)=(\frac{k^{2}}{\log(k)})^{\frac{1}{3}}$ and $a=O(\Delta) = \frac{1}{f(k)}$, we get: $\mathbb{P}^{*} \geq 1 - O(\frac{1}{k})$ and furthermore:
\begin{equation}
Var(\tilde{\theta}^{n}_{p,r}) = \frac{1}{k}\frac{\theta_{p,r}(\pi-\theta_{p,r})}{\pi^{2}} + (\frac{\log(k)}{k^{2}})^{\frac{1}{3}}.
\end{equation}
Thus, from the Chebyshev's inequality, we get the following for every $c>0$ and fixed points $p,r$:
\begin{equation}
\mathbb{P}(|\tilde{\theta}^{n}_{p,r} - \frac{\theta_{p,r}}{\pi}| \geq c (\frac{\sqrt{\log(k)}}{k})^{\frac{1}{3}}) = O(\frac{1}{c^{2}}).
\end{equation}
That completes the proof of Theorem \ref{short_theorem}.

\section{Additional figures}

Figure~\ref{fig:plots2}a and Figure~\ref{fig:plots2}b show how the mean train error is affected by the size of the hash, and  Figure~\ref{fig:plots2}c shows how the mean train error changes with the size of the reduction for the neural network experiment. In Table~\ref{tab:one2} we report both the mean and the standard deviation of the train error across our neural network experiments. \textit{Baseline} refers to the network with one hidden layer containing $100$ hidden units, where all parameters are trained.

\begin{figure}[htp!]
\vspace{-0.60in}
\center
a)\includegraphics[width = 2.6in]{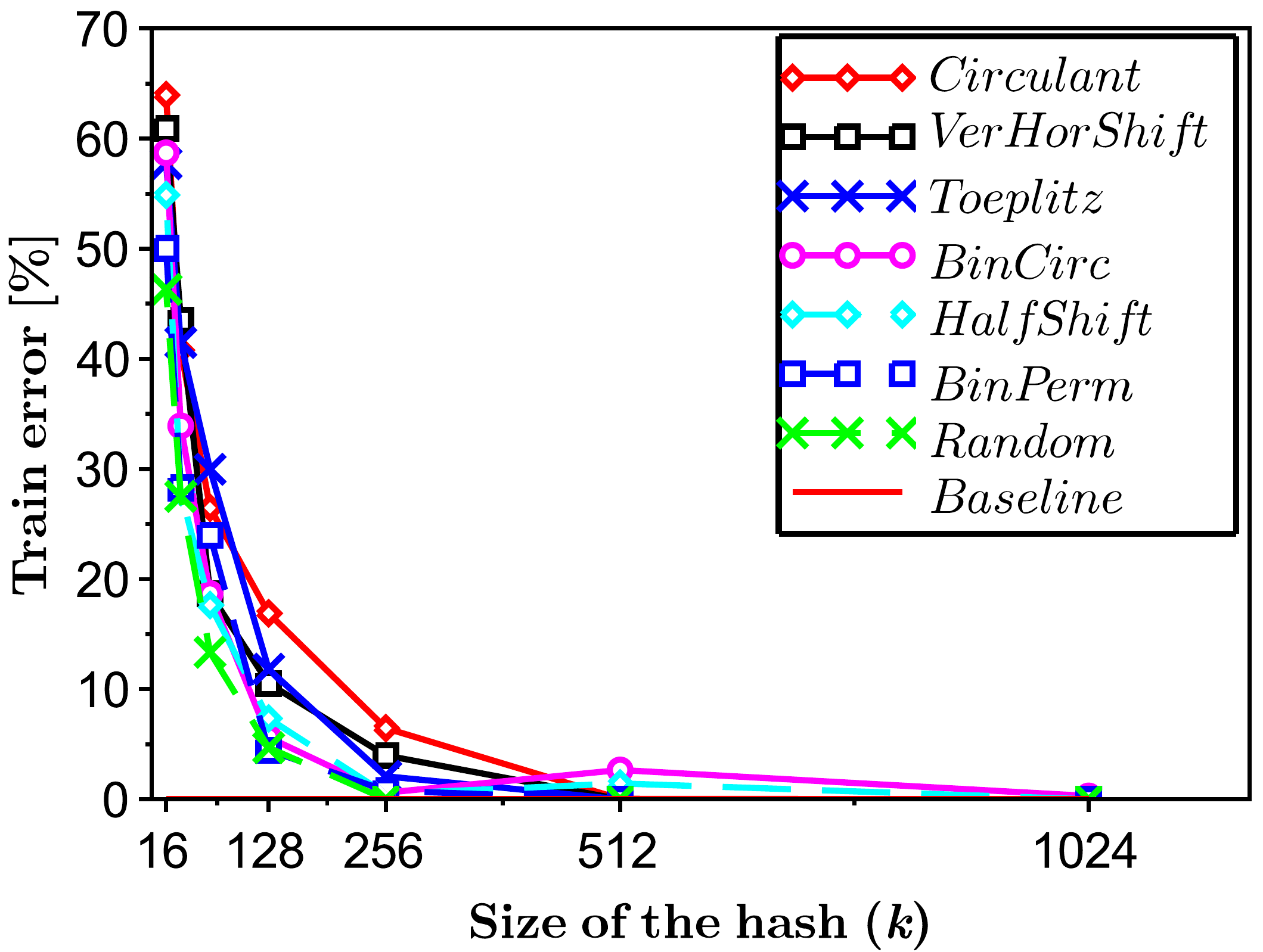}\\
b)\includegraphics[width = 2.6in]{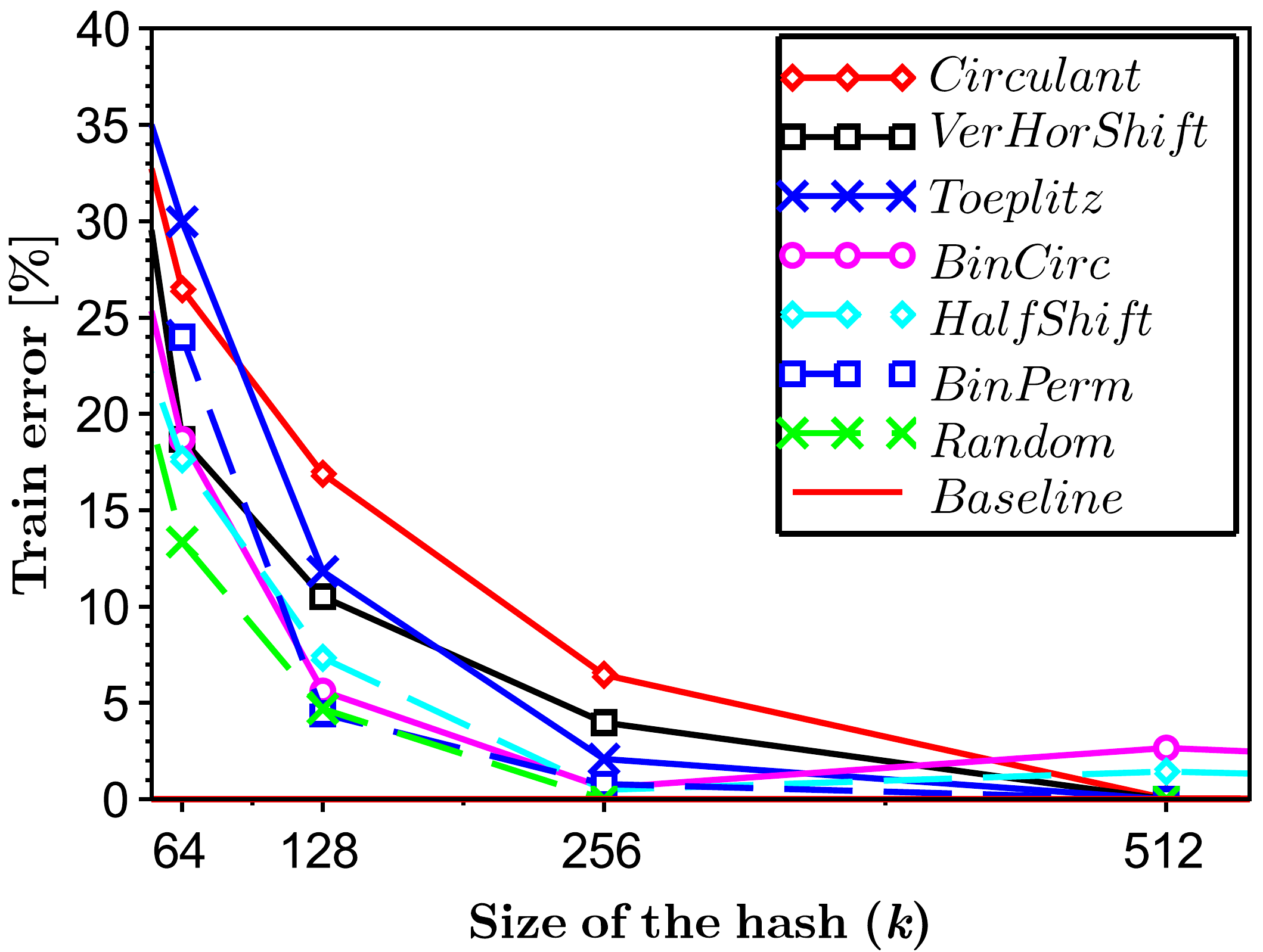}\\
c)\includegraphics[width = 2.6in]{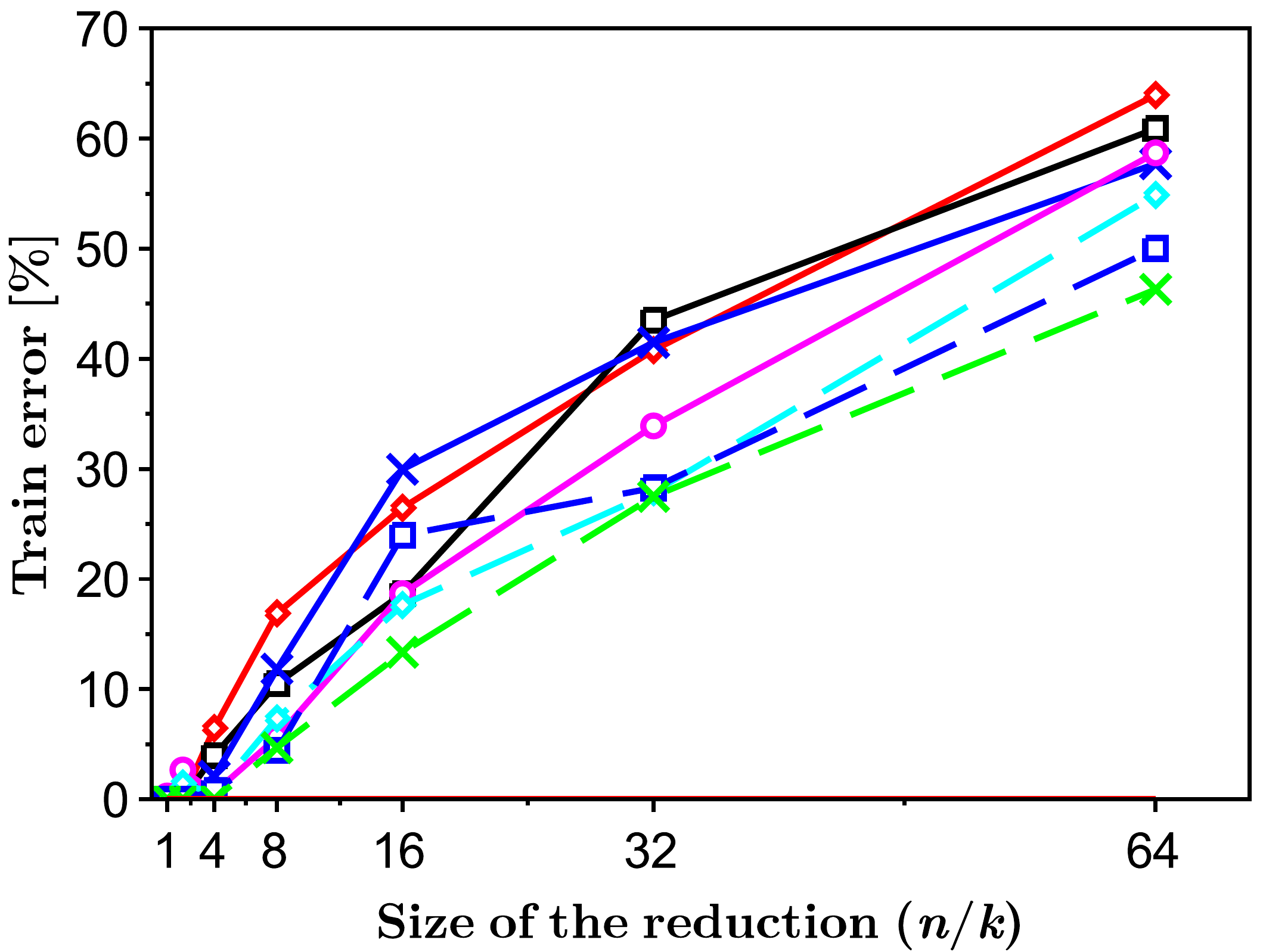}
\vspace{-0.1in}
\caption{Mean train error versus a), b) the size of the hash ($k$), c) the size of the reduction ($n/k$) for the network. b) is a zoomed a). Baseline corresponds to $0\%$.}
\label{fig:plots2}
\vspace{-0.1in}
\end{figure}

\begin{table*}[htp!]
\center
\setlength{\tabcolsep}{2.5pt}
\caption{Mean and std of the test error versus the size of the hash ($k$) / size of the reduction ($n/k$) for $1$-NN.}
\vspace{-0.1in}
\begin{tabular}{|c|c|c|c|c|c|c|c|c|c|c|c|}
\hline
\multicolumn{1}{|c|}{\multirow{2}{*}{$k\:\:$/$\:\:\frac{n}{k}$}} & Circulant & Random & BinPerm & BinCirc & HalfShift & Toeplitz & VerHorShift \\
& $[\%]$ & $[\%]$ & $[\%]$ &  $[\%]$ & $[\%]$ & $[\%]$ & $[\%]$\\
  \hline
\hline
$1024$ / $1$ & $6.02 \pm 0.64$ & $4.83 \pm 0.19$ & $6.67 \pm 0.65$ & $12.77 \pm 2.86$ & $6.38 \pm 0.44$ & $6.22 \pm 1.20$ & $6.30 \pm 0.76$\\
\hline
$512$ / $2$ & $12.98 \pm 11.29$ & $5.77 \pm 0.11$ & $8.15 \pm 0.56$ & $12.40 \pm 2.32$ & $7.25 \pm 0.71$ & $9.11 \pm 2.28$ & $10.81 \pm 4.31$\\
\hline
$256$ / $4$ & $17.73 \pm 6.66$ & $8.51 \pm 0.35$ & $11.11 \pm 1.15$ & $12.13 \pm 4.35$ & $12.05 \pm 2.94$ & $15.66 \pm 3.36$ & $18.19 \pm 5.46$\\
\hline
$128$ / $8$ & $34.80 \pm 14.59$ & $14.44 \pm 0.89$ & $17.20 \pm 2.26$ & $22.15 \pm 6.45$ & $24.74 \pm 8.14$ & $33.90 \pm 13.90$ & $30.37 \pm 7.52$\\
\hline
$64$ / $16$ & $45.91 \pm 5.50$ & $27.57 \pm 1.58$ & $29.53 \pm 3.40$ & $35.33 \pm 5.58$ & $36.58 \pm 10.71$ & $51.10 \pm 13.98$ & $41.66 \pm 8.08$\\
\hline
$32$ / $32$ & $65.06 \pm 9.60$ & $40.58 \pm 2.49$ & $43.58 \pm 4.66$ & $53.05 \pm 5.39$ & $47.18 \pm 7.19$ & $58.24 \pm 8.87$ & $56.73 \pm 6.09$\\
\hline
$16$ / $64$ & $68.61 \pm 5.72$ & $58.72 \pm 3.08$ & $60.30 \pm 6.11$ & $66.29 \pm 4.79$ & $60.84 \pm 5.31$ & $72.50 \pm 6.04$ & $72.50 \pm 5.91$\\
\hline
\end{tabular} 
\label{tab:one3}
\end{table*}

\begin{figure}[h!]
\vspace{-3in}
\center
a)\includegraphics[width = 2.6in]{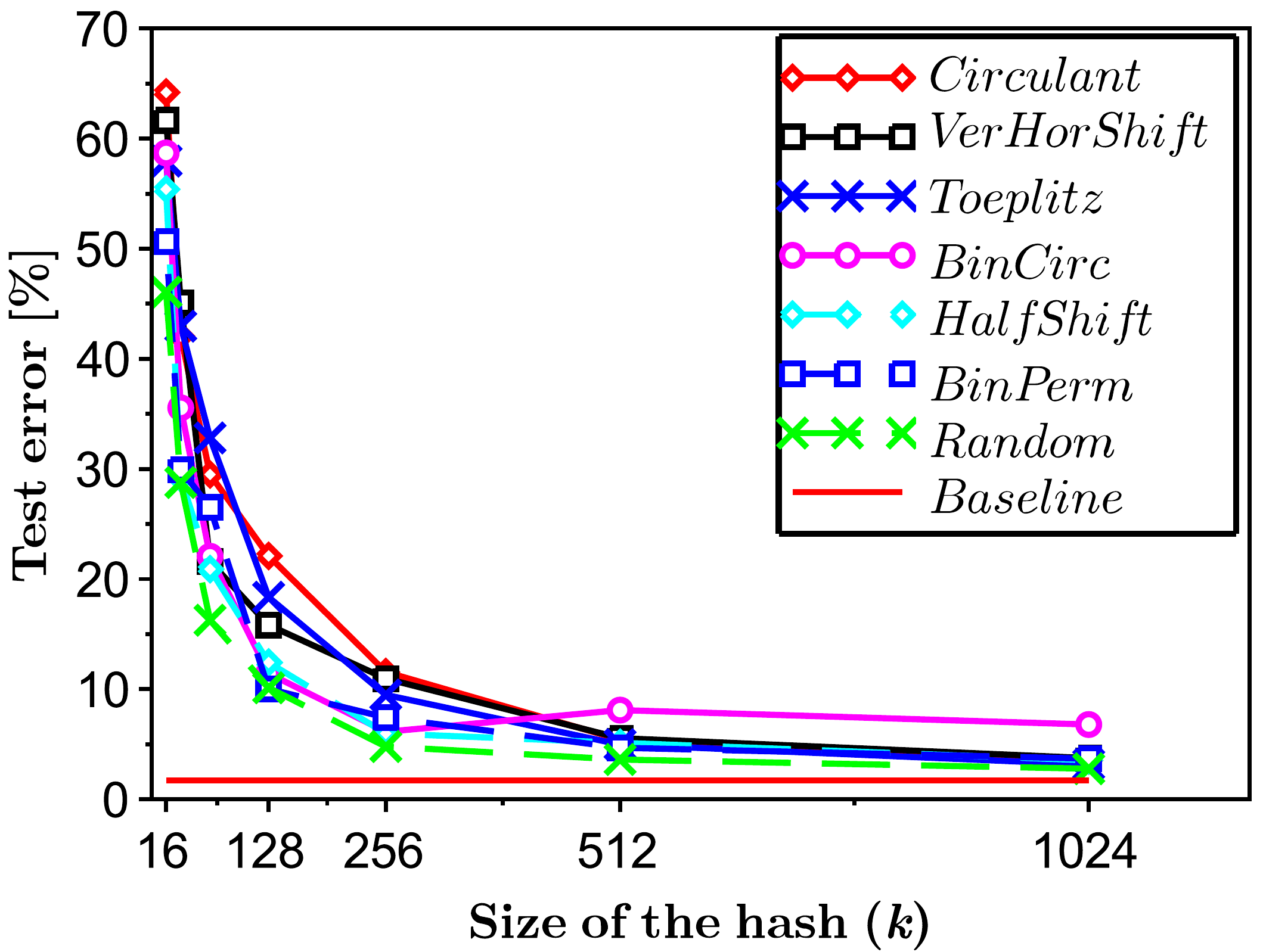}\\
b)\includegraphics[width = 2.6in]{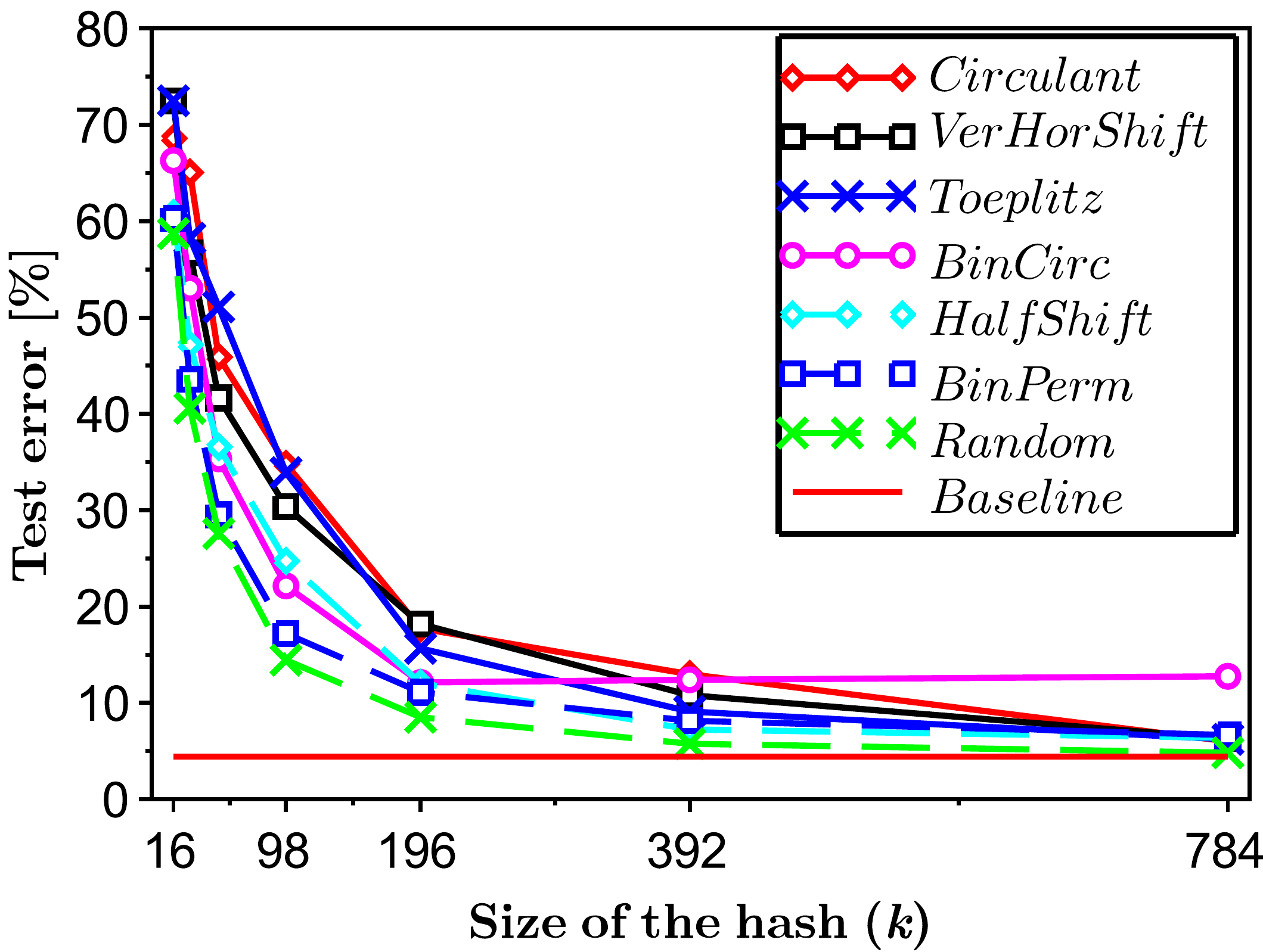}
\vspace{-0.1in}
\caption{Mean test error versus the size of the hash ($k$) (original plot) for a) the network, b) $1$-NN.}
\label{fig:plotsorig}
\end{figure}

Figure~\ref{fig:plotsorig}a shows the original version of Figure~\ref{fig:plots}a (before zoom). Figure~\ref{fig:plotsorig}b shows the original version of Figure~\ref{fig:plotskNN}a (before zoom). Finally, Table~\ref{tab:one3} shows the mean and the standard deviation of the test error versus the size of the hash ($k$)/size of the reduction ($n/k$) for $1$-NN.

\end{document}